\documentclass[letterpaper]{article} 
\usepackage{aaai23}  
\usepackage{times}  
\usepackage{helvet}  
\usepackage{courier}  
\usepackage[hyphens]{url}  
\usepackage{graphicx} 
\usepackage{natbib}  
\usepackage{caption} 
\usepackage{algorithm}
\usepackage{algorithmic}

\usepackage{newfloat}
\usepackage{listings}
\DeclareCaptionStyle{ruled}{labelfont=normalfont,labelsep=colon,strut=off} 
\floatstyle{ruled}
\newfloat{listing}{tb}{lst}{}
\floatname{listing}{Listing}
\title{Learning to Select Pivotal Samples for Meta Re-weighting}
\author{
    Yinjun Wu,
    Adam Stein,
    Jacob Gardner,
    Mayur Naik
}
\affiliations{
    University of Pennsylvania \\

     wuyinjun@seas.upenn.edu, steinad@seas.upenn.edu,  jacobrg@seas.upenn.edu,  mhnaik@cis.upenn.edu
}

\usepackage{bibentry}
\usepackage{amsmath,amssymb}
\usepackage{multibib}
\newcites{supp}{Appendix Reference}
\usepackage{enumitem}
\usepackage{makecell}
\usepackage{multirow}
\usepackage{multicol}
\usepackage{empheq}
\usepackage{nameref}
\usepackage{hhline}
\let\namerefOld\nameref
\renewcommand{\nameref}[1]{\textbf{``\namerefOld{#1}''}}
\usepackage{xcolor}

\newcommand*\widefbox[1]{\fbox{\hspace{0.5em}#1\hspace{0.5em}}}

\newcommand{\revise}[1]{{#1}}
\usepackage{amsthm}
\usepackage{newfloat}
\usepackage{listings}

\newcommand{\x}{\textbf{x}}

\newcommand{\f}[2]{f_{#1}(#2)}
\newcommand{\mf}[2]{f_{\text{meta},#1}(#2)}
\newcommand{\y}{y}
\newcommand{\Dtrain}{D_{\text{train}}}

\newcommand{\centroid}{C}

\newcommand{\concatvec}[1]{G_{#1}}

\newcommand{\mx}[1]{\textbf{x}_{\text{meta}, #1}}
\newcommand{\my}[1]{y_{\text{meta}, #1}}

\newcommand{\Dmeta}{D_{\text{meta}}}

\newcommand{\lastxout}[1]{\textbf{x}^{\text{last out}}_{#1}}

\newcommand{\w}{w}

\newcommand{\bw}{\textbf{W}}

\newcommand{\A}{\textbf{A}}

\newcommand{\vect}[1]{\text{vec}(#1)}

\newcommand{\random}{Random}
\newcommand{\uncertain}{Uncertain}
\newcommand{\certain}{Certain}

\newcommand{\finetune}{Fine-tuning}
\newcommand{\ta}{TA-VAAL}
\newcommand{\craige}{craige}

\newcommand{\cifarone}{CIFAR-10}
\newcommand{\cifar}{CIFAR}
\newcommand{\mnist}{MNIST}
\newcommand{\imagenet}{Imagenet-10}
\newcommand{\cifartwo}{CIFAR-100}

\newtheorem{theorem}{Theorem}
\newtheorem{lemma}{Lemma}

\newcommand{\ourmethodonefull}{Representation-based clustering method}
\newcommand{\ourmethodone}{RBC}
\newcommand{\ourmethodtwofull}{Gradient-based clustering method}
\newcommand{\ourmethodtwo}{GBC}
\definecolor{methodonecolor}{RGB}{0,176,80}
\definecolor{methodtwocolor}{RGB}{112,48,160}
\makeatletter
\renewcommand{\paragraph}{%
  \@startsection{paragraph}{4}%
  {\z@}{0ex \@plus 0ex \@minus .2ex}{-0.5em}%
  {\normalfont\normalsize\bfseries}%
}
\makeatother

\begin{document}

\maketitle

\begin{abstract}
Sample re-weighting strategies provide a promising mechanism to deal with imperfect training data in machine learning, such as noisily labeled or class-imbalanced data.
One such strategy involves formulating a bi-level optimization problem called the {\em meta re-weighting problem}, whose goal is to optimize performance on a small set of perfect pivotal samples, called {\em meta samples}.
Many approaches have been proposed to efficiently solve this problem.
However, all of them assume that a perfect meta sample set is already provided while we observe that the selections of meta sample set is performance-critical. In this paper, we study how to {\em learn} to identify such a meta sample set from a large, imperfect training set, that is subsequently cleaned and used to optimize performance in the meta re-weighting setting.
We propose a learning framework which reduces the meta samples selection problem to a weighted K-means clustering problem through rigorously theoretical analysis.
We propose two clustering methods within our learning framework, \ourmethodonefull\ (\ourmethodone) and \ourmethodtwofull\ (\ourmethodtwo), for balancing performance and computational efficiency.
Empirical studies demonstrate the performance advantage of our methods over various baseline methods.
\end{abstract}

\section{Introduction}
\label{sec: intro}

Recently, with the advent of the data-centric AI era \cite{miranda2021towards, polyzotis2021can, hajij2021data}, there is an increasing concern about the {\em quality} of data for training neural network models.
How to construct and maintain a high-quality data set is extremely challenging due to the existence of various defects in real-life data, e.g., imperfect labels or imbalanced distributions across classes.
To tackle these issues, various techniques have been explored.
One such example is \revise{the} sample re-weighting \revise{strategy} \cite{shu2019meta,ren2018learning,hu2019learning, jiang2018mentornet,chang2017active}, which target\revise{s} jointly learning to obtain {\em re-weighted} training samples and training neural nets upon them.

One promising strategy for learning to re-weight training samples is to leverage the framework of meta learning \cite{hospedales2021meta, andrychowicz2016learning, thrun2012learning} by formulating this problem as a bi-level optimization problem \cite{shu2019meta,ren2018learning,hu2019learning}.
In this approach, the weights of training samples are learned so that the performance of the models learned on the re-weighted training samples is maximized on a small set of perfect samples---referred to as {\em meta samples}.
Existing works mainly focus on designing computationally efficient algorithms for solving this bi-level optimization problem.
For example, \cite{shu2019meta} propose a {\em meta re-weighting} algorithm which alternates between updates to the model parameters and the sample weights.
These algorithms, however, rely on the assumption that the meta sample set is {\em given}, and they construct this set by random sampling in their empirical studies.
However, as the toy example in Figure \ref{fig:toy_boundaries} shows, randomly selected meta samples may perform worse than carefully selected ones by using our methods (62.9\% vs.\ 87.1\% on test accuracy), which we further verify in Section \nameref{sec: experiment}.

\begin{figure*}
    \centering
    \includegraphics[width=0.9\textwidth]{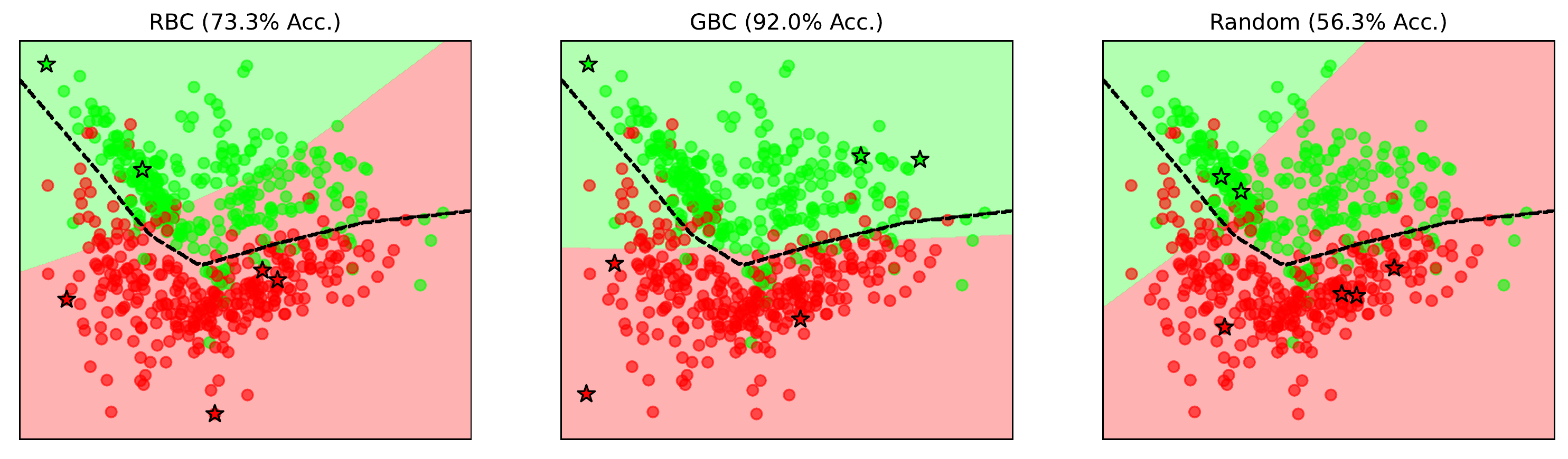}
    \caption{We produce a toy two-dimensional dataset by drawing 1000 samples from a mixture of four Gaussian distributions over two variables where the distributions are centered at the four vertices of the 2-dimensional hypercube. The upper two distributions are labeled green while the lower two are labeled red and 1\% of the labels are flipped to introduce a small amount of noise to the ground truth. We visualize all the samples with their ground-truth labels in this figure. This toy dataset is then divided into 600 training, 240 testing, and 160 validation samples using a random partition, and a randomly selected 60\% of the training labels are flipped. To learn a robust model (which is a neural network with two hidden layers in this example) on this noisy training set, we employ the meta-reweighting algorithm with 6 cleaned meta samples. 
    We then show the selected meta samples (outlined with stars) and the learned decision boundaries (shaded regions) from our methods and the random selection method. The expected classifier, i.e., the learned classifier (black dotted line) on the ``clean'' training set is also visualized. 
    By examining the learned classifier from these methods,
    we see that the one learned by random selections deviates farther from the expected classifier, thus leading to worse prediction performance than our methods (56.3\% vs. 92.0\%).
    }
    \label{fig:toy_boundaries}
\end{figure*}

In this paper, we study how to learn to identify a set of meta samples from a large, imperfect training set such that the {\em meta re-weighting} performance is optimized. Specifically, we propose a framework which reduces the problem of selecting such meta samples to a weighted K-means clustering problem through rigorous theoretical analysis. This derivation basically transforms the formula for iteratively updating sample weights from the meta re-weighting algorithm into a weighted K-means clustering objective function. We can show that \textit{optimizing this objective function can aid in effectively distinguishing high-quality training samples from low-quality ones by giving them more confident sample weights (i.e. weights close to 0 or 1)}. 
This objective function, however, requires the gradients of {\em each individual training sample} as input, which is computationally expensive. To facilitate efficient evaluation of this objective function, we propose two methods, i.e. \ourmethodonefull\ (\ourmethodone) and \ourmethodtwofull\ (\ourmethodtwo), which balance performance with computational efficiency. Specifically, by assuming that the gradients of the bottom layers of the neural nets are insignificant, \ourmethodone\ only utilizes the gradient of the last layer, which is efficiently calculated through feed-forward passes. In contrast, \ourmethodtwo\ samples model parameters such that the estimation of the objective function in the above K-means problem is unbiased. Due to the necessity of explicitly (but partially) computing sample-wise gradients, \ourmethodtwo\ is slower than \ourmethodone, but can lead to better model performance in most cases. 

We further explore whether our methods select reasonable meta samples for re-weighting noisily labeled data and class-imbalanced data by 
conducting experiments on re-weighting \mnist, CIFAR, and \imagenet\ datasets in the presence of noisy labels or imbalanced class distribution. The results show that with the same meta re-weighting algorithm, our methods outperform other sample selection strategies in most cases.




\section{Related Work}\label{sec: related_work}

\paragraph{Sample re-weighting}
The problem of re-weighting training samples for a neural network model has been extensively studied in the literature. Sample re-weighting can be beneficial for constructing robust neural network models in the presence of many defects in training data, such as corrupted labels \cite{han2018co, ren2018learning, shu2019meta}, biased distributions \cite{khan2017cost, dong2017class}, low cardinalities \cite{hu2019learning} and adversarial attacks \cite{holtz2021learning}. Other than solving this problem within the meta-learning framework (e.g., \cite{shu2019meta, ren2018learning, hu2019learning}), various strategies have been proposed for deriving sample weights. For example, in \cite{wang2017robust}, the sample weights are modeled as a Bayesian latent variable and inferred through probabilistic models. In \cite{jiang2018mentornet}, a mentor network is designed to derive the sample weights such that the target model does not overfit on samples with noisy labels, which falls within the curriculum learning \cite{bengio2009curriculum} framework. In \cite{kumar2010self}, the weights of training samples are determined by their training loss during the training process. However, as \cite{shu2019meta} suggests, these re-weighting techniques all perform worse than the meta re-weighting algorithm in the presence of label noise and distribution imbalance in training data. 

\paragraph{Data efficiency}
As mentioned in Section \nameref{sec: intro}, it is critical to obtain large amounts of high-quality training samples for deep neural nets. However, this can be expensive and time consuming since labeling typically requires non-trivial work from human annotators, especially in scientific domains (see e.g, \cite{karimi2020deep, irvin2019chexpert}). High labeling cost is thus a strong motivator for studies on various label efficiency techniques, e.g., active learning (see a survey in \cite{ren2021survey} and some recent works \cite{mirzasoleiman2020coresets}), semi-supervised learning (see \cite{van2020survey}), and weakly-supervised learning (see Snorkel \cite{ratner2017snorkel}) in the past few years. All of these studies aim at minimizing human labeling effort while maintaining relatively high model performance. Note that for the meta re-weighting problem, the construction of perfect meta samples also requires human labeling effort when label noise exists. Therefore, our framework shares the same spirit as the traditional label efficiency research. 


\paragraph{Data valuation}
\revise{In the literature, other than active learning, there exists many techniques to quantify the importance of individual samples, e.g., influence function \cite{koh2017understanding} and its variants \cite{wu2021chef}, Glister \cite{killamsetty2021glister}, HOST-CP \cite{das2021finding}, TracIn \cite{pruthi2020estimating}, DVRL \cite{yoon2020data} and Data Shapley value \cite{ghorbani2019data}. However, among these methods, Data Shapley value \cite{ghorbani2019data} is very computationally expensive while others rely on the assumption that a set of ``clean'' validation samples (or meta samples) are given, which is thus not suitable for our framework (we have more detailed discussions on Data Shapley value and its extensions in Appendix \nameref{appendix: related_work}). We therefore do not include these solutions as baseline methods. }


\section{Background: the meta re-weighting algorithm}\label{sec: background}
In this section, we present some necessary details on the meta re-weighting algorithm from \cite{shu2019meta}. 

Suppose the meta re-weighting algorithm is conducted on a large imperfect training set, $\Dtrain=\{(\x_j,\y_j)\}_{j=1}^N$ and a small perfect meta set $\Dmeta=\{(\mx{i}, \my{i})\}_{i=1}^M$. Imagine that we want to learn a model parameterized by $\Theta$, and the loss evaluated on a training sample $(\x_j,\y_j)$ and a meta sample $(\mx{i}, \my{i})$ is denoted as $\f{j}{\Theta}$ and $\mf{i}{\Theta}$ respectively. We further denote the weight of each training sample $j$ as $\w_j$ (between 0 and 1). Following \cite{shu2019meta}, the meta re-weighting algorithm jointly learns the weights $\bw = \{\w_j\}_{j=1}^N$ and the model parameter $\Theta$ by solving the following bi-level optimization problem:
\begin{small}
\begin{align}\label{eq: meta_reweight_goal}
    \begin{split}
    &\min_{\bw} \frac{1}{M} \sum\nolimits_{i=1}^M \mf{i}{\Theta^*(\bw)},\\
    & \text{s.t.}\ \Theta^*(\bw) = \text{argmin}_{\Theta} \frac{1}{N} \sum\nolimits_{j=1}^N\w_j \f{j}{\Theta},
    \end{split}
\end{align}
\end{small}
in which $\Theta^*(\bw)$ denotes the learned model parameters on the training set weighted by $\bw$. This problem can be efficiently solved by the meta re-weighting algorithm proposed by \cite{shu2019meta}, which can be abstracted with the following formulas \footnote{Note that these formulas are slightly different from the ones in \cite{shu2019meta} since the sample weights in \cite{shu2019meta} are produced by another neural net. But its learning algorithm is also applicable to the case where the sample weights are updated directly. We therefore start from this simple case. \revise{Further note that \cite{ren2018learning} and \cite{hu2019learning} solve Equation \eqref{eq: meta_reweight_goal} in a similar manner. Therefore, although we develop our methods mostly based on \cite{shu2019meta}, 
they are also potentially applicable to the solutions in \cite{ren2018learning} and \cite{hu2019learning}. We therefore discuss how it can be extended to \cite{ren2018learning}, in Appendix \nameref{appendix: generalization_method}}}:
\begin{small}
\begin{empheq}[box=\widefbox]{align}
    & \textbf{Meta re-weighting:} \nonumber \\
    &\hat{\Theta}(\bw_t) = \Theta_t - \frac{\alpha_t}{N}\sum\nolimits_{j=1}^N\w_{j,t} \nabla_{\Theta} \f{j}{\Theta}|_{\Theta = \Theta_t}
    \label{eq: model_pre_update}\\
    & \w_{j,t+1} = \w_{j,t} - \frac{\eta_t}{M}\sum\nolimits_{i =1}^M \nabla_{\w_j} \mf{i}{\hat{\Theta}(\bw_t)} |_{\bw = \bw_t}
    \label{eq: sample_weights_update}\\
    & \Theta_{t+1} = \Theta_t - \frac{\alpha_t}{N} \cdot \sum\nolimits_{j=1}^N \w_{j,t+1}\nabla_{\Theta} \f{j}{\Theta}|_{\Theta = \Theta_t}
    \label{eq: model_post_update}
\end{empheq}
\end{small}
The above formulas show how to update the model parameter and sample weights at the $t^{th}$ iteration. Among these formulas, Equation \eqref{eq: model_pre_update} tries to update the model parameter $\Theta_t$ given the current sample weights $\bw_t=[\w_{1,t},\w_{2,t},\dots,\w_{N,t}]$, which is then employed for updating the sample weights in Equation \eqref{eq: sample_weights_update}. Afterwards, in Equation \eqref{eq: model_post_update}, the updated sample weights, $\bw_{t+1}$, are inserted into Equation \eqref{eq: model_pre_update} to obtain the model parameters for the next iteration, i.e., $\Theta_{t+1}$. This process is then repeated until the convergence. 
\section*{Method}\label{sec: method}


Unlike \cite{shu2019meta, hu2019learning, ren2018learning} where the meta set $\Dmeta$ is assumed to be given, our goal is to select this set from $\Dtrain$. Once this meta set is selected and possibly cleaned by humans (when noisy labels exist), the meta re-weighting algorithm can be used. We hope that the resulting model performance is optimized with respect to the sample selection strategy. We observe that one critical property of such $\Dmeta$ is that it needs to produce \textit{``significant'' cumulative gradient updates (rather than near-zero gradient) in Equation \eqref{eq: sample_weights_update} for every training sample $j (=1,2,\dots,N)$ and every iteration $t$ in the meta re-weighting algorithm}. This can thus guarantee that good training samples are efficiently up-weighted while bad training samples are efficiently down-weighted. 
Therefore, our goal is to maximize the magnitude of the sum of the gradient in Equation \eqref{eq: sample_weights_update} evaluated for each training sample $j$, across all iterations:
\begin{small}
\begin{align}\label{eq: meta_reweighting_goal}
\begin{split}
    & \max_{\Dmeta}\left|\sum\nolimits_{\hat{\Theta}(\bw_t)} 1/M\cdot\sum\nolimits_{i=1}^M \nabla_{\w_j} \mf{i}{\hat{\Theta}(\bw_t)}\right|,\\
    &\text{for all $j=(1,2,\dots,N)$},
\end{split}
\end{align}
\end{small}

which we rewrite as follows according to \cite{shu2019meta} (the constant coefficients are ignored below):

\begin{small}
\begin{align}\label{eq: gradient_inner_prod}
\begin{split}
&  \max_{\Dmeta}\left|\sum_{\hat{\Theta}(\bw_t), \Theta_t}\sum_{i=1}^M \langle\nabla_{\Theta} \mf{i}{\Theta}|_{\Theta = \hat{\Theta}(\bw_t)}, \nabla_{\Theta} \f{j}{\Theta}|_{\Theta = \Theta_t}\rangle\right|,
\end{split}
\end{align}
\end{small}
which thus
represents \textit{the Frobenius inner product of the gradient of the loss} between the meta sample $i$ and the training sample $j$.
If the above inner product is large enough, the weight of this sample will be significantly updated. Since we want to maximize the updates of the weight of each training sample, we sum up the above formula over all training samples, leading to:
\begin{small}
\begin{align*}
\begin{split}
    & \max_{\Dmeta}\sum_{j=1}^N{\left|\sum_{\hat{\Theta}(\bw_t), \Theta_t}\sum_{i=1}^M \langle\nabla_{\Theta} \mf{i}{\Theta}|_{\Theta = \hat{\Theta}(\bw_t)}, \nabla_{\Theta} \f{j}{\Theta}|_{\Theta = \Theta_t}\rangle\right|},
\end{split}
\end{align*}
\end{small}

which can be further approximated as follows by leveraging the fact that $\hat{\Theta}(\bw_t)$, 
is very close to $\Theta_t$:
\begin{small}
\begin{align}\label{eq: our_init_objective0}
    \begin{split}
    \max_{\Dmeta}\sum_{j=1}^N{\left|\sum_{i=1}^M \sum_{\Theta_t}\textcolor{blue}{\langle\nabla_{\Theta} \mf{i}{\Theta}|_{\Theta =\Theta_t}, \nabla_{\Theta} \f{j}{\Theta}|_{\Theta = \Theta_t}\rangle}\right|}
\end{split}
\end{align}
\end{small}

Equation \eqref{eq: our_init_objective0} can be further rewritten as the following Meta-Sample Search Objective (MSSO):
\begin{small}
\begin{align}\label{eq: our_init_objective}
\begin{split}
    & \text{MSSO} := \text{Equation }\eqref{eq: our_init_objective0} = \max_{\Dmeta}\sum_{j=1}^N \left|\sum_{i=1}^M \langle\concatvec{j}, \concatvec{\text{meta}, i}\rangle\right|,
\end{split}
\end{align}
\end{small}
in which, we define
$\concatvec{j} = \left[\concatvec{j}^{(1)}, \dots,\concatvec{j}^{(t)},\dots\right]$ and
$\concatvec{\text{meta},i} = \left[\concatvec{\text{meta},i}^{(1)}, \dots,\concatvec{\text{meta},i}^{(t)},\dots\right]$
as block matrices formed by concatenating the gradients,
$\concatvec{\text{meta},i}^{(t)} := \nabla_{\Theta} \mf{i}{\Theta}|_{\Theta = \Theta_t}$ and
$\concatvec{j}^{(t)} := \nabla_{\Theta} \f{j}{\Theta}|_{\Theta = \Theta_t}$,
from each iteration into one matrix.

Note that the meta sample set, $\Dmeta$, needs to be selected from the training set, $\Dtrain$. Thus, explicitly solving MSSO
is computationally intractable since there are ${N\choose M}$ possible selections of a meta set of size $M$.
In what follows, we present an approximation to MSSO with rigorous guarantees, \revise{which can be effectively solved with a weighted K-means clustering algorithm.}

\subsection*{Approximating MSSO}
We 
show that with reasonable assumptions, solving MSSO is approximately equivalent to searching for a set of cluster centroids, $\mathcal{C} = \{\centroid_i\}_{i=1}^M$, i.e.,:
\begin{small}
\begin{align*}
\begin{split}
    &\text{MSSO} 
    \approx \max_{\mathcal{C}}\sum\nolimits_{j=1}^N\left| \sum\nolimits_{i=1}^M \langle\concatvec{j}, \centroid_i\rangle\right|\\
\end{split}
\end{align*}
\end{small}

which can be approximated by solving the following $M$-clustering objective (MCO)
\begin{small}
\begin{align}\label{eq: transformed_k_means_objective}
\begin{split}
    &\text{MSSO} 
    \approx \text{MCO}:= \max_{\mathcal{C}}\sum\nolimits_{j=1}^N \sum\nolimits_{i=1}^M \left|\langle\concatvec{j}, \centroid_i\rangle\right|,
\end{split}
\end{align}
\end{small}

where MSSO is approximated by moving the absolute value to the inside of the sum.
The approximation above can be justified by the following Theorem.

\begin{theorem}\label{theorem: main}
Suppose that for each sample $i$, the positive terms in the innermost sum of Equation \eqref{eq: our_init_objective} are dominant over the negative terms or vice versa, i.e.:
\begin{small}
\begin{align*}
& \frac{|\sum_{\langle\concatvec{j}, \centroid_i\rangle > 0}\langle\concatvec{j}, \centroid_i\rangle|}{|\sum_{\langle\concatvec{j}, \centroid_i\rangle < 0}\langle\concatvec{j}, \centroid_i\rangle|} > D \gg 1, \\
&\text{or } \frac{|\sum_{\langle\concatvec{j}, \centroid_i\rangle < 0}\langle\concatvec{j}, \centroid_i\rangle|}{|\sum_{\langle\concatvec{j}, \centroid_i\rangle > 0}\langle\concatvec{j}, \centroid_i\rangle|} > D \gg 1, \text{for all $i$},
\end{align*}
\end{small}
then solving MCO is a $\frac{D-1}{D+1}$-approximation to solving MSSO, i.e., $\frac{D-1}{D+1} \leq \frac{MSSO}{MCO} \leq 1$

\end{theorem}

The proof is included in Appendix \nameref{sec: theorem_main}. Intuitively, we can see that our approximation is perfect when each inner product in \eqref{eq: our_init_objective} is positive, and we have less of a guarantee of the effectiveness when a cluster is less homogeneous in the sign of the inner products between its members and centroid. Indeed, we found that the assumptions in the above theorem hold in most cases (see Appendix \nameref{sec: supple_exp}).
Therefore, due to the closeness of MSSO and MCO, we focus on solving MCO rather than MSSO.

\begin{figure*}
	\centering
	\includegraphics[
	width=0.9\textwidth
	]{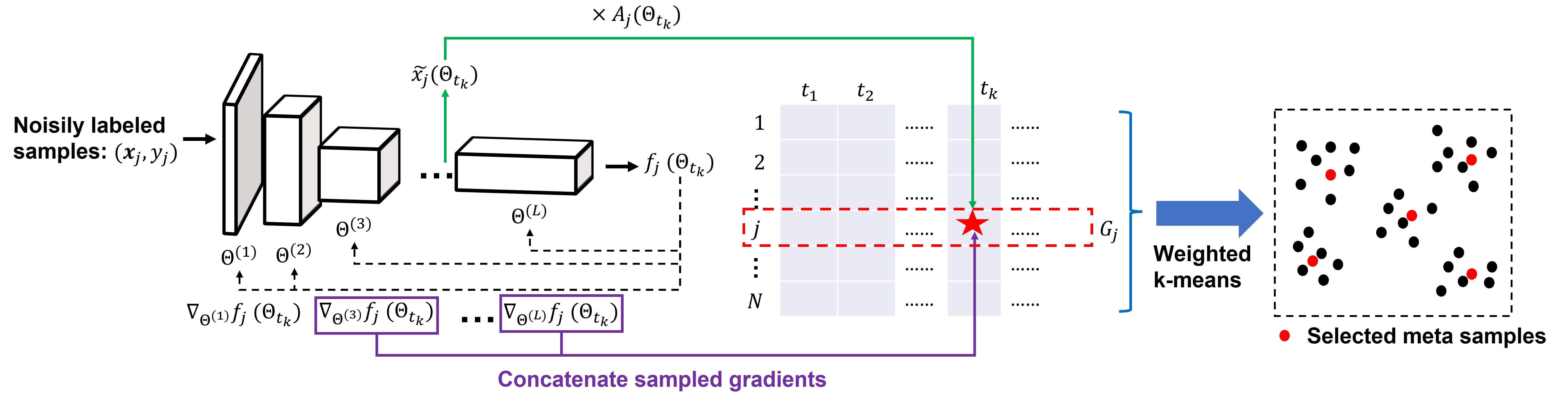}
	\captionof{figure}{Overview of our methods, \ourmethodone\ and \ourmethodtwo. We use \textcolor{methodonecolor}{$\rightarrow$} (green colored arrow) and \textcolor{methodtwocolor}{$\rightarrow$} (purple colored arrow) to denote the data flow of \ourmethodone\ and \ourmethodtwo\ respectively. Specifically, at each sampled time step $t_k$, for each input training sample $(\x_j,\y_j)$, \ourmethodone\ combines its feature vector from the input to the last layer of the model, $\Tilde{\x_j}$, and the coefficient, $\A_{j}(\Theta_{t_k})$ (defined in Equation \eqref{eq: method_one_coeff0}) while \ourmethodtwo\ concatenate the gradients from the sampled layers in the model. We then concatenate the above calculated results from all the time steps $t_1,t_2,\dots$ to compose the input to weighted K-means clustering algorithm, $G_j$ (see the red dotted box), which is then used for determining the meta samples.}\label{fig: method_overview}
\end{figure*}

\subsection*{Solving MCO}\label{sec: solve_mco}

MCO resembles the K-means clustering objective, so it is promising to solve it with the K-means clustering algorithm. As the first step toward this, MCO is transformed to the following form:
\begin{small}
\begin{align}\label{eq: weighted_k_means_obj_0}
    \begin{split}
        \text{MCO} &=\max_{\mathcal{C}}\sum_{j=1}^N \|\concatvec{j}\| \sum_{i=1}^M \|\centroid_i\|\cdot |\text{cosine}(\concatvec{j}, \centroid_i)|,
    \end{split}
\end{align}
\end{small}
which can be regarded as a weighted K-means clustering objective function.
Specifically, 
the norm of each $\centroid_i$ is used for re-weighting the cosine similarity between each training sample $j$ and each cluster centroid $i$, which is followed by re-weighting the overall similarity of each training sample $j$ to \textit{all} cluster centroids with the norm of $\concatvec{j}$.
Further details on how to tailor the vanilla K-means clustering algorithm to solve MCO are presented in Appendix \nameref{appendix: k_means}. 
After $\mathcal{C} = \{\centroid_i\}_{i=1}^M$ is identified by this weighted K-means algorithm, the samples closest to each cluster centroid are returned as the selected meta samples, $\Dmeta$\footnote{We notice that other strategies, e.g., \cite{auvolat2015clustering}, can be employed to solve MCO, which, however\revise{,} do not \revise{perform} well and are thus ignored. }. 
Note that in Equation \eqref{eq: transformed_k_means_objective}, collecting all $\concatvec{j}$ is very expensive.
This is because $j$ is over all training samples which can be very large,
and $\concatvec{j}$ depends on all $\Theta_t$, i.e., the model parameters at all iterations (see Equation \eqref{eq: our_init_objective}). 

To address the above efficiency concerns, we firstly propose two methods, i.e., \ourmethodonefull\ (\ourmethodone) and \ourmethodtwofull\ (\ourmethodtwo) in Section \nameref{sec: method_one} and Section \nameref{sec: method_two} respectively, for addressing the first concern. We further discuss how to sample from all $\Theta_t (t=1,2,\dots)$ in Section \nameref{sec: chicken_egg_problem} to handle the second concern.

\subsubsection{\ourmethodonefull\ (\ourmethodone)}\label{sec: method_one}
\ourmethodone\ is built upon the assumption that 
the gradient of the model parameters on the bottom layers (i.e. those layers closer to the input) is less significant than the ones in the last layer. Due to the vanishing gradient problem, this assumption usually holds in practice. As a consequence, we 
only consider the gradients from the last layer in \text{Equation }\eqref{eq: transformed_k_means_objective}, leading to the following approximations on $\concatvec{j}$:
\begin{small}
\begin{align}\label{eq: method_one_objective_function}
    \begin{split}
     \concatvec{j} = \A_{j}(\Theta_t)\Tilde{\x_j}(\Theta_t)^\top,
\end{split}
\end{align}
\end{small}
in which 
$\Tilde{\x_j}(\Theta_t)$ represents the input to the last linear layer in the neural network model produced by the 
training sample $j$, while $\A_{j}(\Theta_t)$
is defined as follows:
\begin{small}
\begin{align}\label{eq: method_one_coeff0}
    & \A_{j}(\Theta_t) = \text{softmax}(\Theta_t^{(-1)}\Tilde{\x_j}(\Theta_t)) - \text{onehot}(\y_j)
\end{align}
\end{small}
in which $\Theta_t^{(-1)}$ represents the model parameters in the last layer. The detailed derivation of Equation \eqref{eq: method_one_objective_function} is included in Appendix \nameref{appedix:derivation_method_one_objective}. Equation \eqref{eq: method_one_objective_function}-\eqref{eq: method_one_coeff0} shows that to obtain $\concatvec{j}$, 
only forward passes on the models are needed, which makes this method very efficient.


\subsubsection{\ourmethodtwofull\ (\ourmethodtwo)}\label{sec: method_two}
Unlike \ourmethodone, \ourmethodtwo\ is applicable to general cases where the gradients generated by the bottom neural layers may be significant. To facilitate efficient evaluations of MCO, we importance sample the network layers from the model, such that we can obtain an unbiased estimation of Equation \eqref{eq: transformed_k_means_objective}. Then $\concatvec{j}$ is constructed by concatenating the gradients calculated in those sampled layers.



Specifically, first of all, the blue part of Equation \eqref{eq: our_init_objective0} (which is the essential part of Equation \eqref{eq: transformed_k_means_objective}) can be rewritten in terms of a sum over the model parameters at each layer $l \in [1,2,\dots, L]$, i.e.:
\begin{small}
\begin{align}\label{eq: inner_prod_by_layer}
\begin{split}
&\langle\nabla_{\Theta} \mf{i}{\Theta}|_{\Theta = \Theta_t},  \nabla_{\Theta} \f{j}{\Theta}|_{\Theta = \Theta_t}\rangle\\
& = \left[\sum\nolimits_{l=1}^L \langle\nabla_{\Theta^{(l)}} \mf{i}{\Theta}), \nabla_{\Theta^{(l)}} \f{j}{\Theta})\rangle\right]_{\Theta = \Theta_t},
\end{split}    
\end{align}
\end{small}
in which $\Theta^{(l)}$ represents the model parameters at the $l^{th}$ layer. Then the above formula could be rewritten as follows:
\begin{small}
\begin{align}\label{eq: inner_prod_by_layer_by_importance_score}
    \begin{split}
        &\text{Equation }\eqref{eq: inner_prod_by_layer} \\
        & = A\cdot\left[\sum_{l=1}^L \frac{A^{(l)}}{A} \langle\frac{\nabla_{\Theta^{(l)}} \mf{i}{\Theta}}{\sqrt{A^{(l)}}}, \frac{\nabla_{\Theta^{(l)}} \f{j}{\Theta}}{\sqrt{A^{(l)}}}\rangle\right]_{\Theta = \Theta_t},
    \end{split}
\end{align}
\end{small}
in which, $A^{(l)} = \|\frac{1}{N}\sum_{j=1}^N \nabla_{\Theta^{(l)}} \f{j}{\Theta}\|_F^2$ and $A=\sum_{l=1}^L A^{(l)}$

Then we can conduct importance sampling (with replacement) on the $L$ innermost sums in Equation \eqref{eq: inner_prod_by_layer_by_importance_score} for several times (say 5 times)\footnote{we conduct the importance sampling once for all the samples so that the dimension of $\concatvec{j}$ is the same among all the samples. Although it is not rigorously correct, the empirical studies show that this approximation could achieve good performance}, in which the probability of selecting the $l^{th} (l=1,2,\dots,L)$ term is $A^{(l)}/A$. This leads to an unbiased estimation of Equation \eqref{eq: inner_prod_by_layer_by_importance_score} and significant speed-ups. 

\subsubsection{Sampling model parameters from history}\label{sec: chicken_egg_problem}
It is worth noting that $\Theta_t$ is unknown before we obtain all meta samples (see Equation \eqref{eq: model_pre_update}-\eqref{eq: model_post_update}), but it is essential for determining the meta samples (see Equation \eqref{eq: our_init_objective0}). Therefore, we propose to cache the model parameter $\Tilde{\Theta}_t (t=1,\dots, T)$ during the training process without any available meta samples, which is regarded as an approximation of $\Theta_t$. 

In addition, as mentioned above, $\concatvec{j}$ depends on the model parameters from all the time steps, which is thus very expensive to evaluate. We uniformly sample several time steps, instead of using all $\Tilde{\Theta}_t$, to get an unbiased estimation of MCO.

In the end, we visually present both \ourmethodone\ and \ourmethodtwo\ equipped with this sampling technique in Figure \ref{fig: method_overview} and include their pseudo-code in Algorithm \ref{alg: rbc_gbc} in Appendix \nameref{appendix: adapted_k_means}. 

\section{Applications}
We demonstrate the effectiveness of our methods for two applications, i.e., re-weighting a training set with noisy labels and re-weighting an imbalanced training set. In what follows, we discussed how to tailor \ourmethodone\ and \ourmethodtwo\ to these two applications.

\subsection{Re-weighting a training set with noisy labels}\label{sec: application_label_noise}
To re-weight a noisily labeled training set, we can select a subset of meta samples from the training set and obtain their clean labels from human annotators. 
Note that for \ourmethodone\ and \ourmethodtwo, 
the evaluation of the gradients depend on the clean labels of the meta samples while these clean labels are obtained from human annotators {\em after} \ourmethodone\ or \ourmethodtwo\ is invoked. To address this chicken or the egg issue, we observe that if the loss function is the cross-entropy loss, then the sample-wise gradient, $\nabla_{\Theta} \f{j}{\Theta}$, can be broken into two parts, the {\em label-free part} and the {\em label-dependent part}.
Due to the unavailability of the clean labels, we therefore only leverage the \textit{label-free part} as the input to \ourmethodone\ and \ourmethodtwo.

Although we only use the label-free part, in Appendix \nameref{appendix: label_free_gradient} (see Theorem \ref{theorem: approximate_label_free}), we theoretically analyze under what conditions the label-dependent part is insignificant to determining which cluster each training sample belongs to after the weighted k-means clustering algorithm is invoked.
Those conditions are satisfied by a large portion of the training samples through our empirical studies (see Appendix \nameref{sec: supple_exp}), thus justifying the effectiveness of discarding the label-dependent part.

\subsection{Re-weighting a training set with a class-imbalance}
As indicated by \cite{shu2019meta}, the meta re-weighting algorithm can also be leveraged for re-weighting class-imbalanced training sets. 
Unlike the case where the labels are noisy, we assume clean labels in the class-imbalanced training set. As a consequence, we evaluate the sample-wise gradient $\nabla_{\Theta} \f{j}{\Theta}$ as a whole rather than removing the label-dependent part from it.




\section{Experiments}\label{sec: experiment}

We demonstrate the effectiveness of our methods for training deep neural nets on image classification datasets, \mnist\ \cite{deng2012mnist}, \cifarone\ \cite{krizhevsky2009learning} and \cifartwo\ \cite{krizhevsky2009learning}, and \revise{\imagenet\ \cite{russakovsky2015imagenet}\footnote{\imagenet\ is a subset of ImageNet and produced by following \cite{li2021contrastive}}}.
By following \cite{shu2019meta} and \cite{ren2018learning}, we consider the occurrence of noisy labels and class imbalance respectively on the training set. All the code is publicly available\footnote{\url{https://github.com/thuwuyinjun/meta_sample_selections}}.

\begin{table*}[t]
\caption{Test accuracy on \mnist, \cifarone\ and \cifartwo\ dataset with noise rate 60\%}\label{table: noise_label}
\small
\centering

\begin{tabular}[!h]{>{\arraybackslash}p{1.5cm}|>{\centering\arraybackslash}p{1.5cm}|>{\centering\arraybackslash}p{1.35cm}>{\centering\arraybackslash}p{1.5cm}|>{\centering\arraybackslash}p{1.35cm}>{\centering\arraybackslash}p{1.4cm}|>{\centering\arraybackslash}p{0.8cm}>{\centering\arraybackslash}p{1.3cm}} \hline\hline
 Dataset &\multicolumn{1}{c|}{\mnist}&\multicolumn{2}{c|}{\cifarone}&\multicolumn{2}{c}{\cifartwo}&\multicolumn{2}{|c}{\makecell{\imagenet}} \\ \hline
Noise type &adversarial&uniform&adversarial&uniform&adversarial&uniform&adversarial\\ \hline
Base model & 51.74$\pm$1.52 & 77.74$\pm$1.22 & 40.24$\pm$0.39 &43.63$\pm$2.30 & 27.15$\pm$0.40 & 72.22&38.00 \\
\random &85.67$\pm$0.90 & 73.56$\pm$0.40 & 76.02$\pm$2.01 & 42.30$\pm$4.68 &  45.33$\pm$1.70 &93.33 & 59.77\\
\certain &81.84$\pm$0.89 & 74.76$\pm$1.07 &70.78$\pm$5.00 & 45.95$\pm$4.20 & 47.06$\pm$2.10 &91.20 &58.22\\
\uncertain &76.38$\pm$0.54 & 73.83$\pm$0.24 & 74.45$\pm$6.10& 36.67$\pm$0.20 & 44.65$\pm$0.65 &85.22 & 51.00\\
\finetune & 53.39$\pm$1.22& 78.46$\pm$2.10 &23.07$\pm$7.58 &25.28$\pm$1.13 & 24.88$\pm$1.10 & 70.44&35.67\\
\ta &79.31$\pm$0.23 &72.89$\pm$0.82 & 61.46$\pm$4.65&31.07$\pm$2.56&38.79$\pm$0.86&86.34 &43.26 \\
\craige & 92.84$\pm$0.14&78.77$\pm$0.86 &78.55$\pm$1.03 &39.85$\pm$1.23&44.61$\pm$1.21&88.90 & 61.00\\
\ourmethodone-K &\underline{93.78$\pm$0.61}&77.91$\pm$1.43&75.71$\pm$1.22&49.32$\pm$0.35&49.51$\pm$0.43& 91.33&54.67 \\ \hline
\ourmethodone &93.00$\pm$1.01 & \underline{80.15$\pm$0.25} &\underline{79.20$\pm$0.64} & \underline{49.56$\pm$0.53} & \underline{50.60$\pm$1.51}&\textbf{94.22}&\textbf{63.67} \\
\ourmethodtwo &\textbf{94.26$\pm$0.24} &\textbf{80.36$\pm$0.96}&\textbf{80.88$\pm$1.46}& \textbf{50.88$\pm$1.90} & \textbf{53.14$\pm$1.33}&\underline{94.00}&\textbf{63.67} \\ \hline\hline
\end{tabular}
\end{table*}

\subsection{Experimental set-up}\label{sec: experiment_setup}
For the \mnist\ dataset, we train a LeNet model \cite{lecun1998gradient} and for \cifarone, \cifartwo\ and  \imagenet\ dataset, we train a ResNet-34 model \cite{he2016deep}. 
All the hyper-paremters are reported in Appendix \nameref{sec: supple_exp}.
\begin{table}[t]
		\centering
		\begin{small}
\begin{tabular}[!h]{>{\arraybackslash}p{1.5cm}|>{\centering\arraybackslash}p{1.6cm}>{\centering\arraybackslash}p{1.6cm}} \hline\hline
 Dataset &\multicolumn{1}{c}{\cifarone}&\multicolumn{1}{c}{\cifartwo} \\ \hline
BaseModel  & 61.45$\pm$0.60& 28.14$\pm$0.57 \\
\random    & 65.96$\pm$1.74 & 29.29$\pm$0.46 \\
\uncertain & 64.46$\pm$1.20 & 28.39$\pm$0.21 \\
\certain  & 66.05$\pm$1.19 & 28.52$\pm$0.15  \\
\finetune\ & 60.04$\pm$1.69 & 29.73$\pm$0.06\\
\ta & 61.58$\pm$1.21  & 30.89$\pm$1.09 \\
\craige &66.60$\pm$0.89 & 29.56$\pm$1.46\\
\ourmethodone-K& 65.96$\pm$1.02 &30.77$\pm$1.23 \\ \hline
\ourmethodone\ & \textbf{68.18$\pm$1.58} & \underline{31.78$\pm$1.10} \\
\ourmethodtwo\ & \underline{67.37$\pm$1.51} & \textbf{33.87$\pm$0.66} \\\hline\hline
\end{tabular}
\hfill
\caption{Test performance on imbalanced \cifarone\ and \cifartwo\ dataset with imbalanced factor 200}\label{Table: cifar_imb}
\end{small}
\end{table}

\begin{figure}
		\centering
		\includegraphics[
		width=0.35\textwidth
		]{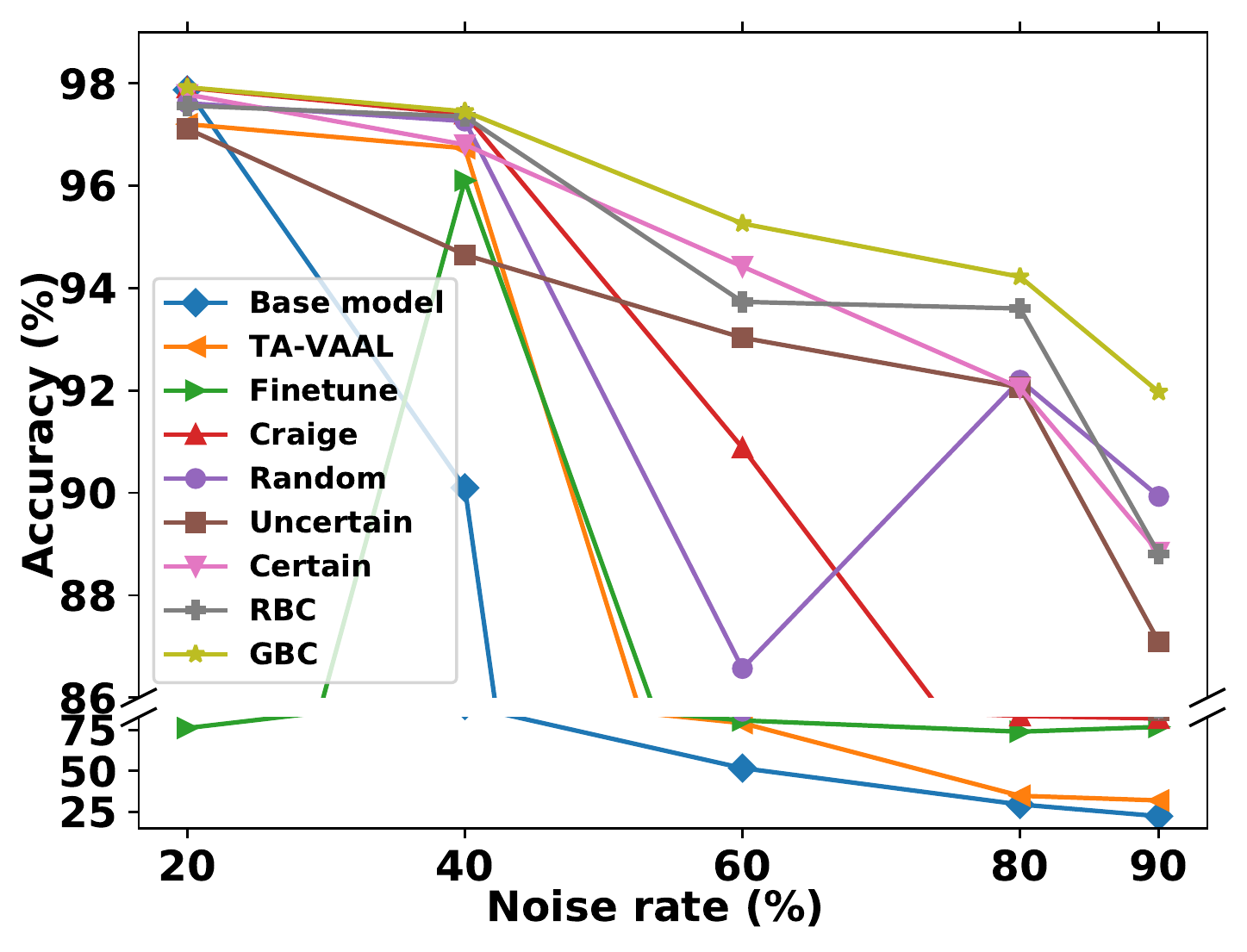}
		\captionof{figure}{Test performance on \mnist\ dataset with varied noisy rate}\label{fig: mnist_vary_noise}
\end{figure}

\subsection{Noisy label experiments}

We first study how our methods perform in the presence of two types of label noise, i.e., {\em uniform noise} and {\em adversarial noise}: 
\begin{itemize}[leftmargin=*, noitemsep,topsep=0pt]
\item {\em uniform noise}: all labels can be uniformly flipped at random to any other label with probability $p/100$, in which $p$ is a percent specified by users. This has been explored in \cite{shu2019meta}  and \cite{ren2018learning};
    \item {\em adversarial noise}:
    the labels for a subset of samples, chosen at random, are determinisitcally mapped to another label (e.g., selected samples with label 0 are all given label 1).
    This is meant to simulate an extreme case where the labels are {\em adversarially} flipped and has been explored in some prior works (e.g., \cite{li2022selective})
\end{itemize}

We present the results with one fixed noise rate, $p = 60$, for both types of noise on \mnist, \cifar\ and \imagenet\ and \revise{the effect of varied $p$ is also explored on the \mnist\ dataset (see Figure \ref{fig: mnist_vary_noise})}. Surprisingly, we found that 60\% uniform noise only reduces the model accuracy on \mnist\ by a few percent. We therefore only report the results on \mnist\ with adversarial noise. 
Throughout this experiment, we compare \ourmethodone\ and \ourmethodtwo\ against the following baseline methods:
\begin{itemize}[leftmargin=*, noitemsep,topsep=0pt]
    \item \textbf{Random selection} (\textit{\random}): We uniformly at random select meta samples from the training set;
    \item \textbf{\finetune}: We fine-tune the model using only the selected meta samples, selected by \random;
    \item \textbf{Active learning}: We select meta samples using 1) \textbf{Uncertainty based selection} (\textit{\uncertain}) \cite{lewis1994sequential} 
    by selecting the {\em most uncertain} training samples, 2) \textbf{Certainty based selection} (\textit{\certain}) 
    by selecting the {\em most certain} training samples and 3) two state-of-the-art active learning solutions, \textbf{Task-Aware Variational Adversarial Active Learning}  (\textit{\ta}) \cite{kim2021task} and \textbf{\craige} (\textit{\craige}) \cite{mirzasoleiman2020coresets}
    \item \textbf{\ourmethodone-k}: We use the original K-means clustering algorithm rather than the weighted version proposed in Section \nameref{sec: solve_mco} to determine the meta samples in \ourmethodone.
\end{itemize}

Note that for both of our methods and the above baseline methods, the labels of the selected meta samples are cleaned by human annotators, which is simulated by replacing their noisy labels with ground-truth labels. This thus justifies the use of the perfectly labeled benchmark datasets (rather than real datasets with unreliable labels). As a result, for fair comparison, our methods and the above baseline methods share the same labeling budget, which is set as 20, 50, 200 and 50 for \mnist, \cifarone, \cifartwo\ and \imagenet\ respectively (\revise{which includes the labeled samples in the pre-training phase}). We pre-train the models by running the meta re-weighting algorithm with small amount of randomly selected meta samples (10 for \mnist, 10 for \cifarone, 50 for \cifartwo\ and 10 for \imagenet) since selecting those meta samples one time leads to sub-optimal performance\footnote{We empirically show this in Appendix \nameref{sec: supple_exp}}, \revise{which is applied to all the baseline methods for fair comparison}. 
\begin{table}
\caption{The AUC score of the sample weights on \mnist\ with noise rate 80\%}\label{Table: mnist_sample_weights}
\small
\centering
\begin{tabular}[!h]{>{\arraybackslash}p{1cm}|>{\centering\arraybackslash}p{1cm}>{\centering\arraybackslash}p{1cm}} \hline\hline
Method &All& Boundary \\ \hline
\random  & 0.922   & 0.589 \\
\ourmethodone\ & \textbf{0.958} &0.775 \\
\ourmethodtwo\ & 0.949 & \textbf{0.854}\\\hline\hline
\end{tabular}
\end{table}
\paragraph{Overall performance} We present the test accuracy in Table \ref{table: noise_label}\footnote{We report the validation accuracy for \imagenet\ since the ground-truth labels of test samples are invisible} after running the meta re-weighting algorithm with meta samples selected by different methods. As indicated by this table, the clustering-based methods, \ourmethodone-k, \ourmethodone\ and \ourmethodtwo\ can significantly outperform other methods in most cases and the performance gains are up to~6\% (see the performance difference between \ourmethodtwo\ and \certain\ in column ``adversarial'' of \cifartwo\ dataset). 
\revise{Furthermore, \ourmethodone\ consistently outperforms \ourmethodone-K, which suggests the weighted K-means algorithm is capable of identifying a better set of meta samples than the original K-means algorithm.}

\paragraph{Efficiency of \ourmethodone} \revise{We also observe a trade-off between performance and speed when comparing \ourmethodtwo\ and \ourmethodone.  According to Table \ref{table: noise_label}, \ourmethodtwo\ performs better than \ourmethodone\ in most cases while the former is slower than the latter (2.5 hours VS 3 mins) to construct $\concatvec{j}$. \revise{Note that the running time of \ourmethodone\ is negligible in comparison to the running time of the meta re-weighting algorithm, which is around 4 mins per epoch and there are hundreds of epochs in total.}}

\paragraph{Robustness against varied noise rate} \revise{As indicated by Figure \ref{fig: mnist_vary_noise}, both \ourmethodone\ and \ourmethodtwo\ outperform all the baseline methods across all the noise rates and the performance gains become even larger with more samples being noisily labeled (up to 2\%). This indicates the robustness of our methods against a varied level of label noise.}

\paragraph{Sample weight distributions} \revise{Recall that our methods depend on the assumption that larger updates to the sample weights will more effectively result in the weights of the noisy and clean samples approaching 0 and 1 respectively. We therefore empirically verify this assumption by inspecting the sample weights learned by \random, \ourmethodone\ and \ourmethodtwo. Specifically, we calculate the AUC between the learned sample weights and the cleanness of the sample labels (1 for clean while 0 for corrupt). We report this quantity for \mnist\ with 80\% noisy labels in Table \ref{Table: mnist_sample_weights} for the entire training set and for the 1000 samples nearest to the decision boundary\footnote{We measure the distance between each sample and the decision boundary by utilizing the metric proposed by \cite{elsayed2018large}}. As Table \ref{Table: mnist_sample_weights} shows, the AUC of \ourmethodone\ and \ourmethodtwo\ are significantly higher than that of \random, especially for those samples near the boundary, thus suggesting the capability for \ourmethodone\ and \ourmethodtwo\ to better distinguish between clean and noisy samples. This could thus explain why \ourmethodone\ and \ourmethodtwo\ achieve superior performance according to Table \ref{fig: mnist_vary_noise}, thereby verifying our assumption.}


\subsection{Class imbalance experiments}

For evaluating our method on class imbalanced data, we follow \cite{cui2019class} to produce the long-tailed \cifar\ dataset. Specifically, we down-sample some classes so that the ratio between the number of training samples in the largest class and that in the smallest one (which is denoted the {\em imbalance factor}) is large. \revise{In Table \ref{Table: cifar_imb}, we report the results with imbalance factor 200 on \cifarone\ and \cifartwo\ dataset}. 
As shown in Table \ref{Table: cifar_imb}, our method, \ourmethodone, outperforms all the baseline methods for \cifarone\ and \cifartwo\, and the performance gain is up to 3.10\%.










\subsection{Other experimental results} 
\revise{Due to the space limit, all other experimental results are presented in Appendix: \nameref{sec: supple_exp}, including the experiments with real labeling noise on \cifar\ dataset, the effect of the pre-training phase, the effect of varied number of meta samples, the effect of the number of sampled gradients in \ourmethodone\ and \ourmethodtwo\ (recall that both approximate Equation \eqref{eq: weighted_k_means_obj_0} through sampling according to Section: \nameref{sec: solve_mco}), and some qualitative studies.}

\section{Conclusion}
In this work, we propose a clustering-based framework for selecting pivotal samples to improve performance of meta re-weighting in the presence of various defects on training data. Based on our theoretical analysis, we show that selecting pivotal samples can be reduced to a weighted K-means algorithm under reasonable assumptions. To efficiently evaluate this algorithm we propose two methods, \ourmethodone\ and \ourmethodtwo, which can balance the computational efficiency and prediction performance.
Through empirical studies on noisily labeled and class-imbalanced image classification benchmark datasets, we can demonstrate that 
our technique could select a better set of pivotal samples for meta re-weighting algorithm than other sample selection techniques, thereby resulting in better model performance.
\newpage
\section{Acknowledgments}
We thank our anonymous reviewers for valuable feedback. This research was supported by grants from DARPA (\#FA8750-19-2-0201) and NSF award (IIS-2145644).
\bibliography{reference}

\begin{thebibliography}{43}
\providecommand{\natexlab}[1]{#1}

\bibitem[{Andrychowicz et~al.(2016)Andrychowicz, Denil, Gomez, Hoffman, Pfau,
  Schaul, Shillingford, and De~Freitas}]{andrychowicz2016learning}
Andrychowicz, M.; Denil, M.; Gomez, S.; Hoffman, M.~W.; Pfau, D.; Schaul, T.;
  Shillingford, B.; and De~Freitas, N. 2016.
\newblock Learning to learn by gradient descent by gradient descent.
\newblock \emph{Advances in neural information processing systems}, 29.

\bibitem[{Auvolat et~al.(2015)Auvolat, Chandar, Vincent, Larochelle, and
  Bengio}]{auvolat2015clustering}
Auvolat, A.; Chandar, S.; Vincent, P.; Larochelle, H.; and Bengio, Y. 2015.
\newblock Clustering is efficient for approximate maximum inner product search.
\newblock \emph{arXiv preprint arXiv:1507.05910}.

\bibitem[{Bengio et~al.(2009)Bengio, Louradour, Collobert, and
  Weston}]{bengio2009curriculum}
Bengio, Y.; Louradour, J.; Collobert, R.; and Weston, J. 2009.
\newblock Curriculum learning.
\newblock In \emph{Proceedings of the 26th annual international conference on
  machine learning}, 41--48.

\bibitem[{Chang, Learned-Miller, and McCallum(2017)}]{chang2017active}
Chang, H.-S.; Learned-Miller, E.; and McCallum, A. 2017.
\newblock Active bias: Training more accurate neural networks by emphasizing
  high variance samples.
\newblock \emph{Advances in Neural Information Processing Systems}, 30.

\bibitem[{Cui et~al.(2019)Cui, Jia, Lin, Song, and Belongie}]{cui2019class}
Cui, Y.; Jia, M.; Lin, T.-Y.; Song, Y.; and Belongie, S. 2019.
\newblock Class-balanced loss based on effective number of samples.
\newblock In \emph{Proceedings of the IEEE/CVF conference on computer vision
  and pattern recognition}, 9268--9277.

\bibitem[{Das et~al.(2021)Das, Singh, Chatterjee, Bhattacharya, and
  Bhattacharya}]{das2021finding}
Das, S.; Singh, A.; Chatterjee, S.; Bhattacharya, S.; and Bhattacharya, S.
  2021.
\newblock Finding High-Value Training Data Subset Through Differentiable Convex
  Programming.
\newblock In \emph{Joint European Conference on Machine Learning and Knowledge
  Discovery in Databases}, 666--681. Springer.

\bibitem[{Deng(2012)}]{deng2012mnist}
Deng, L. 2012.
\newblock The mnist database of handwritten digit images for machine learning
  research [best of the web].
\newblock \emph{IEEE signal processing magazine}, 29(6): 141--142.

\bibitem[{Dong, Gong, and Zhu(2017)}]{dong2017class}
Dong, Q.; Gong, S.; and Zhu, X. 2017.
\newblock Class rectification hard mining for imbalanced deep learning.
\newblock In \emph{Proceedings of the IEEE International Conference on Computer
  Vision}, 1851--1860.

\bibitem[{Elsayed et~al.(2018)Elsayed, Krishnan, Mobahi, Regan, and
  Bengio}]{elsayed2018large}
Elsayed, G.; Krishnan, D.; Mobahi, H.; Regan, K.; and Bengio, S. 2018.
\newblock Large margin deep networks for classification.
\newblock \emph{Advances in neural information processing systems}, 31.

\bibitem[{Ghorbani and Zou(2019)}]{ghorbani2019data}
Ghorbani, A.; and Zou, J. 2019.
\newblock Data shapley: Equitable valuation of data for machine learning.
\newblock In \emph{International Conference on Machine Learning}, 2242--2251.
  PMLR.

\bibitem[{Hajij et~al.(2021)Hajij, Zamzmi, Ramamurthy, and
  Saenz}]{hajij2021data}
Hajij, M.; Zamzmi, G.; Ramamurthy, K.~N.; and Saenz, A.~G. 2021.
\newblock Data-Centric AI Requires Rethinking Data Notion.
\newblock \emph{arXiv preprint arXiv:2110.02491}.

\bibitem[{Han et~al.(2018)Han, Yao, Yu, Niu, Xu, Hu, Tsang, and
  Sugiyama}]{han2018co}
Han, B.; Yao, Q.; Yu, X.; Niu, G.; Xu, M.; Hu, W.; Tsang, I.; and Sugiyama, M.
  2018.
\newblock Co-teaching: Robust training of deep neural networks with extremely
  noisy labels.
\newblock \emph{Advances in neural information processing systems}, 31.

\bibitem[{He et~al.(2016)He, Zhang, Ren, and Sun}]{he2016deep}
He, K.; Zhang, X.; Ren, S.; and Sun, J. 2016.
\newblock Deep residual learning for image recognition.
\newblock In \emph{Proceedings of the IEEE conference on computer vision and
  pattern recognition}, 770--778.

\bibitem[{Holtz, Weng, and Mishne(2021)}]{holtz2021learning}
Holtz, C.; Weng, T.-W.; and Mishne, G. 2021.
\newblock Learning Sample Reweighting for Adversarial Robustness.

\bibitem[{Hospedales et~al.(2021)Hospedales, Antoniou, Micaelli, and
  Storkey}]{hospedales2021meta}
Hospedales, T.~M.; Antoniou, A.; Micaelli, P.; and Storkey, A.~J. 2021.
\newblock Meta-Learning in Neural Networks: A Survey.
\newblock \emph{IEEE Transactions on Pattern Analysis and Machine
  Intelligence}.

\bibitem[{Hu et~al.(2019)Hu, Tan, Salakhutdinov, Mitchell, and
  Xing}]{hu2019learning}
Hu, Z.; Tan, B.; Salakhutdinov, R.~R.; Mitchell, T.~M.; and Xing, E.~P. 2019.
\newblock Learning data manipulation for augmentation and weighting.
\newblock \emph{Advances in Neural Information Processing Systems}, 32.

\bibitem[{Irvin et~al.(2019)Irvin, Rajpurkar, Ko, Yu, Ciurea-Ilcus, Chute,
  Marklund, Haghgoo, Ball, Shpanskaya et~al.}]{irvin2019chexpert}
Irvin, J.; Rajpurkar, P.; Ko, M.; Yu, Y.; Ciurea-Ilcus, S.; Chute, C.;
  Marklund, H.; Haghgoo, B.; Ball, R.; Shpanskaya, K.; et~al. 2019.
\newblock Chexpert: A large chest radiograph dataset with uncertainty labels
  and expert comparison.
\newblock In \emph{Proceedings of the AAAI conference on artificial
  intelligence}, volume~33, 590--597.

\bibitem[{Jiang et~al.(2018)Jiang, Zhou, Leung, Li, and
  Fei-Fei}]{jiang2018mentornet}
Jiang, L.; Zhou, Z.; Leung, T.; Li, L.-J.; and Fei-Fei, L. 2018.
\newblock Mentornet: Learning data-driven curriculum for very deep neural
  networks on corrupted labels.
\newblock In \emph{International Conference on Machine Learning}, 2304--2313.
  PMLR.

\bibitem[{Karimi et~al.(2020)Karimi, Dou, Warfield, and
  Gholipour}]{karimi2020deep}
Karimi, D.; Dou, H.; Warfield, S.~K.; and Gholipour, A. 2020.
\newblock Deep learning with noisy labels: Exploring techniques and remedies in
  medical image analysis.
\newblock \emph{Medical Image Analysis}, 65: 101759.

\bibitem[{Khan et~al.(2017)Khan, Hayat, Bennamoun, Sohel, and
  Togneri}]{khan2017cost}
Khan, S.~H.; Hayat, M.; Bennamoun, M.; Sohel, F.~A.; and Togneri, R. 2017.
\newblock Cost-sensitive learning of deep feature representations from
  imbalanced data.
\newblock \emph{IEEE transactions on neural networks and learning systems},
  29(8): 3573--3587.

\bibitem[{Killamsetty et~al.(2021)Killamsetty, Sivasubramanian, Ramakrishnan,
  and Iyer}]{killamsetty2021glister}
Killamsetty, K.; Sivasubramanian, D.; Ramakrishnan, G.; and Iyer, R. 2021.
\newblock Glister: Generalization based data subset selection for efficient and
  robust learning.
\newblock In \emph{Proceedings of the AAAI Conference on Artificial
  Intelligence}, volume~35, 8110--8118.

\bibitem[{Kim et~al.(2021)Kim, Park, Kim, and Chun}]{kim2021task}
Kim, K.; Park, D.; Kim, K.~I.; and Chun, S.~Y. 2021.
\newblock Task-aware variational adversarial active learning.
\newblock In \emph{Proceedings of the IEEE/CVF Conference on Computer Vision
  and Pattern Recognition}, 8166--8175.

\bibitem[{Koh and Liang(2017)}]{koh2017understanding}
Koh, P.~W.; and Liang, P. 2017.
\newblock Understanding black-box predictions via influence functions.
\newblock In \emph{International conference on machine learning}, 1885--1894.
  PMLR.

\bibitem[{Krizhevsky, Hinton et~al.(2009)}]{krizhevsky2009learning}
Krizhevsky, A.; Hinton, G.; et~al. 2009.
\newblock Learning multiple layers of features from tiny images.

\bibitem[{Kumar, Packer, and Koller(2010)}]{kumar2010self}
Kumar, M.; Packer, B.; and Koller, D. 2010.
\newblock Self-paced learning for latent variable models.
\newblock \emph{Advances in neural information processing systems}, 23.

\bibitem[{LeCun et~al.(1998)LeCun, Bottou, Bengio, and
  Haffner}]{lecun1998gradient}
LeCun, Y.; Bottou, L.; Bengio, Y.; and Haffner, P. 1998.
\newblock Gradient-based learning applied to document recognition.
\newblock \emph{Proceedings of the IEEE}, 86(11): 2278--2324.

\bibitem[{Lewis and Gale(1994)}]{lewis1994sequential}
Lewis, D.~D.; and Gale, W.~A. 1994.
\newblock A sequential algorithm for training text classifiers.
\newblock In \emph{SIGIR’94}, 3--12. Springer.

\bibitem[{Li et~al.(2022)Li, Xia, Ge, and Liu}]{li2022selective}
Li, S.; Xia, X.; Ge, S.; and Liu, T. 2022.
\newblock Selective-supervised contrastive learning with noisy labels.
\newblock In \emph{Proceedings of the IEEE/CVF Conference on Computer Vision
  and Pattern Recognition}, 316--325.

\bibitem[{Li et~al.(2021)Li, Hu, Liu, Peng, Zhou, and Peng}]{li2021contrastive}
Li, Y.; Hu, P.; Liu, Z.; Peng, D.; Zhou, J.~T.; and Peng, X. 2021.
\newblock Contrastive clustering.
\newblock In \emph{Proceedings of the AAAI Conference on Artificial
  Intelligence}, volume~35, 8547--8555.

\bibitem[{Miranda(2021)}]{miranda2021towards}
Miranda, L.~J. 2021.
\newblock Towards data-centric machine learning: a short review.
\newblock \emph{Ljvmiranda921. Github. Io}.

\bibitem[{Mirzasoleiman, Bilmes, and
  Leskovec(2020)}]{mirzasoleiman2020coresets}
Mirzasoleiman, B.; Bilmes, J.; and Leskovec, J. 2020.
\newblock Coresets for data-efficient training of machine learning models.
\newblock In \emph{International Conference on Machine Learning}, 6950--6960.
  PMLR.

\bibitem[{Polyzotis and Zaharia(2021)}]{polyzotis2021can}
Polyzotis, N.; and Zaharia, M. 2021.
\newblock What can Data-Centric AI Learn from Data and ML Engineering?
\newblock \emph{arXiv preprint arXiv:2112.06439}.

\bibitem[{Pruthi et~al.(2020)Pruthi, Liu, Kale, and
  Sundararajan}]{pruthi2020estimating}
Pruthi, G.; Liu, F.; Kale, S.; and Sundararajan, M. 2020.
\newblock Estimating training data influence by tracing gradient descent.
\newblock \emph{Advances in Neural Information Processing Systems}, 33:
  19920--19930.

\bibitem[{Ratner et~al.(2017)Ratner, Bach, Ehrenberg, and
  R{\'e}}]{ratner2017snorkel}
Ratner, A.~J.; Bach, S.~H.; Ehrenberg, H.~R.; and R{\'e}, C. 2017.
\newblock Snorkel: Fast training set generation for information extraction.
\newblock In \emph{Proceedings of the 2017 ACM international conference on
  management of data}, 1683--1686.

\bibitem[{Ren et~al.(2018)Ren, Zeng, Yang, and Urtasun}]{ren2018learning}
Ren, M.; Zeng, W.; Yang, B.; and Urtasun, R. 2018.
\newblock Learning to reweight examples for robust deep learning.
\newblock In \emph{International conference on machine learning}, 4334--4343.
  PMLR.

\bibitem[{Ren et~al.(2021)Ren, Xiao, Chang, Huang, Li, Gupta, Chen, and
  Wang}]{ren2021survey}
Ren, P.; Xiao, Y.; Chang, X.; Huang, P.-Y.; Li, Z.; Gupta, B.~B.; Chen, X.; and
  Wang, X. 2021.
\newblock A survey of deep active learning.
\newblock \emph{ACM Computing Surveys (CSUR)}, 54(9): 1--40.

\bibitem[{Russakovsky et~al.(2015)Russakovsky, Deng, Su, Krause, Satheesh, Ma,
  Huang, Karpathy, Khosla, Bernstein et~al.}]{russakovsky2015imagenet}
Russakovsky, O.; Deng, J.; Su, H.; Krause, J.; Satheesh, S.; Ma, S.; Huang, Z.;
  Karpathy, A.; Khosla, A.; Bernstein, M.; et~al. 2015.
\newblock Imagenet large scale visual recognition challenge.
\newblock \emph{International journal of computer vision}, 115(3): 211--252.

\bibitem[{Shu et~al.(2019)Shu, Xie, Yi, Zhao, Zhou, Xu, and Meng}]{shu2019meta}
Shu, J.; Xie, Q.; Yi, L.; Zhao, Q.; Zhou, S.; Xu, Z.; and Meng, D. 2019.
\newblock Meta-weight-net: Learning an explicit mapping for sample weighting.
\newblock \emph{Advances in neural information processing systems}, 32.

\bibitem[{Thrun and Pratt(2012)}]{thrun2012learning}
Thrun, S.; and Pratt, L. 2012.
\newblock \emph{Learning to learn}.
\newblock Springer Science \& Business Media.

\bibitem[{Van~Engelen and Hoos(2020)}]{van2020survey}
Van~Engelen, J.~E.; and Hoos, H.~H. 2020.
\newblock A survey on semi-supervised learning.
\newblock \emph{Machine Learning}, 109(2): 373--440.

\bibitem[{Wang, Kucukelbir, and Blei(2017)}]{wang2017robust}
Wang, Y.; Kucukelbir, A.; and Blei, D.~M. 2017.
\newblock Robust probabilistic modeling with bayesian data reweighting.
\newblock In \emph{International Conference on Machine Learning}, 3646--3655.
  PMLR.

\bibitem[{Wu, Weimer, and Davidson(2021)}]{wu2021chef}
Wu, Y.; Weimer, J.; and Davidson, S.~B. 2021.
\newblock CHEF: a cheap and fast pipeline for iteratively cleaning label
  uncertainties.
\newblock \emph{Proceedings of the VLDB Endowment}, 14(11): 2410--2418.

\bibitem[{Yoon, Arik, and Pfister(2020)}]{yoon2020data}
Yoon, J.; Arik, S.; and Pfister, T. 2020.
\newblock Data valuation using reinforcement learning.
\newblock In \emph{International Conference on Machine Learning}, 10842--10851.
  PMLR.

\end{thebibliography}


\begin{thebibliography}{12}
\providecommand{\natexlab}[1]{#1}
\providecommand{\url}[1]{\texttt{#1}}
\expandafter\ifx\csname urlstyle\endcsname\relax
  \providecommand{\doi}[1]{doi: #1}\else
  \providecommand{\doi}{doi: \begingroup \urlstyle{rm}\Url}\fi

\bibitem[Ghorbani and Zou(2019)]{ghorbani2019data}
Amirata Ghorbani and James Zou.
\newblock Data shapley: Equitable valuation of data for machine learning.
\newblock In \emph{International Conference on Machine Learning}, pages
  2242--2251. PMLR, 2019.

\bibitem[Henderson et~al.(1983)Henderson, Pukelsheim, and
  Searle]{henderson1983history}
Harold~V Henderson, Friedrich Pukelsheim, and Shayle~R Searle.
\newblock On the history of the kronecker product.
\newblock \emph{Linear and Multilinear Algebra}, 14\penalty0 (2):\penalty0
  113--120, 1983.

\bibitem[Jia et~al.()Jia, Dao, Wang, Hubis, Gurel, Zhang, and
  Song]{jia12efficient}
Ruoxi Jia, David Dao, Boxin Wang, Frances~Ann Hubis, Nezihe~Merve Gurel,
  Bo~Li4~Ce Zhang, and Costas Spanos1~Dawn Song.
\newblock Efficient task-specific data valuation for nearest neighbor
  algorithms.
\newblock \emph{Proceedings of the VLDB Endowment}, 12\penalty0 (11).

\bibitem[Jia et~al.(2019)Jia, Dao, Wang, Hubis, Hynes, G{\"u}rel, Li, Zhang,
  Song, and Spanos]{jia2019towards}
Ruoxi Jia, David Dao, Boxin Wang, Frances~Ann Hubis, Nick Hynes, Nezihe~Merve
  G{\"u}rel, Bo~Li, Ce~Zhang, Dawn Song, and Costas~J Spanos.
\newblock Towards efficient data valuation based on the shapley value.
\newblock In \emph{The 22nd International Conference on Artificial Intelligence
  and Statistics}, pages 1167--1176. PMLR, 2019.

\bibitem[Sim et~al.(2020)Sim, Zhang, Chan, and Low]{sim2020collaborative}
Rachael Hwee~Ling Sim, Yehong Zhang, Mun~Choon Chan, and Bryan Kian~Hsiang Low.
\newblock Collaborative machine learning with incentive-aware model rewards.
\newblock In \emph{International Conference on Machine Learning}, pages
  8927--8936. PMLR, 2020.

\bibitem[Sim et~al.(2022)Sim, Xu, and Low]{sim2022data}
Rachael Hwee~Ling Sim, Xinyi Xu, and Bryan Kian~Hsiang Low.
\newblock Data valuation in machine learning:“ingredients”, strategies, and
  open challenges.
\newblock In \emph{Proc. IJCAI}, 2022.

\bibitem[Stewart(2011)]{stewart2011calculus}
James Stewart.
\newblock \emph{Calculus: Concepts and contexts}.
\newblock Brooks/Cole Pacific Grove, 2011.

\bibitem[Van~der Maaten and Hinton(2008)]{van2008visualizing}
Laurens Van~der Maaten and Geoffrey Hinton.
\newblock Visualizing data using t-sne.
\newblock \emph{Journal of machine learning research}, 9\penalty0 (11), 2008.

\bibitem[Wei et~al.(2021)Wei, Zhu, Cheng, Liu, Niu, and Liu]{wei2021learning}
Jiaheng Wei, Zhaowei Zhu, Hao Cheng, Tongliang Liu, Gang Niu, and Yang Liu.
\newblock Learning with noisy labels revisited: A study using real-world human
  annotations.
\newblock In \emph{International Conference on Learning Representations}, 2021.

\bibitem[Xu et~al.(2021{\natexlab{a}})Xu, Lyu, Ma, Miao, Foo, and
  Low]{xu2021gradient}
Xinyi Xu, Lingjuan Lyu, Xingjun Ma, Chenglin Miao, Chuan~Sheng Foo, and Bryan
  Kian~Hsiang Low.
\newblock Gradient driven rewards to guarantee fairness in collaborative
  machine learning.
\newblock \emph{Advances in Neural Information Processing Systems},
  34:\penalty0 16104--16117, 2021{\natexlab{a}}.

\bibitem[Xu et~al.(2021{\natexlab{b}})Xu, Wu, Foo, and Low]{xu2021validation}
Xinyi Xu, Zhaoxuan Wu, Chuan~Sheng Foo, and Bryan Kian~Hsiang Low.
\newblock Validation free and replication robust volume-based data valuation.
\newblock \emph{Advances in Neural Information Processing Systems},
  34:\penalty0 10837--10848, 2021{\natexlab{b}}.

\bibitem[Yan and Procaccia(2021)]{yan2021if}
Tom Yan and Ariel~D Procaccia.
\newblock If you like shapley then you’ll love the core.
\newblock In \emph{Proceedings of the AAAI Conference on Artificial
  Intelligence}, volume~35, pages 5751--5759, 2021.

\end{thebibliography}

\newpage
\clearpage
\onecolumn
\appendix
\def\theequation{S\arabic{equation}}
\setcounter{equation}{0}

\section{Appendix: more related work}\label{appendix: related_work}

\subsection{Extra related work on Shapley-value based data valuation}
Due to the space limit of the main paper, we provide here a more extensive discussion on the existing data valuation literature, in particular, the works based on Data Shapley value \citesupp{ghorbani2019data}. We notice that \citesupp{sim2022data} summarized the recent progress in this area.
However, as noted by \citesupp{sim2022data}, these solutions are not scalable to large datasets due to the high computational overhead of data Shapley. It is worth noting that although these solutions, e.g., \citesupp{xu2021validation, sim2020collaborative, xu2021gradient}, claim that they are applicable to realistic large datasets, they only study the data Shapley value of several ``partitions'' of the entire dataset and the number of partitions is typically smaller than 10. In contrast, to identify meta samples for the meta re-weighting algorithm, we need to compute the value of each individual training sample and thus the number of ``partitions'' is equivalent to the number of training samples, which is typically very large. This thus indicates that all the existing data Shapley dependent solutions are computationally intractable for our problem. Although there have been many recent attempts to approximately but efficiently compute data Shapley values, e.g., \citesupp{jia2019towards, yan2021if, jia12efficient}, they are still far from being practical solutions since they either explicitly assume that the models have certain properties, which may not hold for general neural nets (e.g., \citesupp{jia2019towards, jia12efficient}), or still require repetitive training (e.g., \citesupp{yan2021if}). As a consequence, it is still an open challenge to efficiently compute data Shapley values for \textit{each sample} in large datasets for \textit{general neural nets} \citesupp{sim2022data}.

In addition, note that \citesupp{xu2021validation} proposes a data valuation metric without relying on the performance on the validation set, which, however, is built upon the data Shapley value, thus suffering from the efficiency issue as mentioned above.


\section{Appendix: Extra algorithmic details}
\subsection{Supplemental materials on the weighted K-means algorithm}\label{appendix: k_means}
In section \nameref{sec: solve_mco}, we discussed adapting the vanilla K-means clustering algorithm for solving MCO, which is presented in details in this section.


\begin{algorithm}

    \textbf{Input}: {A set of gradient vectors $\{\concatvec{j}\}_{j=1}^N$}\\
    \textbf{output}: {A set of cluster centroids $\{\centroid_i\}_{i=1}^M$}
    \begin{algorithmic}[1]
    \STATE Randomly initialize $\{\centroid_i\}_{i=1}^M$
    \WHILE{not converged}
    
        \STATE \textbf{Assignment step:} Assign the sample $j$ to the cluster $\hat{i}$ such that $\hat{i} = \text{argmax}_{i}\{\|\centroid_i\|\text{cosine}(\concatvec{j}, \centroid_i)\}_{i=1}^M$.
        
        \STATE \textbf{Update step:} Update the cluster centroid $j$ by Equation \eqref{eq: weighted_k_means_update_step}.
    
    \ENDWHILE
    
\caption{A weighted K-means clustering algorithm for solving MCO}\label{alg: weighted_k_means}
\end{algorithmic}
\end{algorithm}

\begin{algorithm}
    \textbf{Input:}{A training set $\Dtrain=\{(\x_i,\y_i)\}_{i=1}^N$, the total number of training iterations, $T$, the number of the randomly sampled meta samples in the warm-up phase, $M_0$, and a model with model parameter $\Theta$}\\
    \textbf{Output:}{A set of meta samples $\Dmeta$}
    \begin{algorithmic}
    
    \STATE Initialize $\Dmeta=\{\}$
    
    \COMMENT{Warm-up phase}
    \STATE Randomly sample $M_0$ meta samples and add them into $\Dmeta$
    
    \STATE Run meta re-weighting algorithm with $\Dmeta$ as the meta sample set, 
    resulting in a list of model parameters $\{\Tilde{\Theta}_t\}_{t=1}^T$ at each iteration.

    \WHILE{more validation samples are needed}
    
    \STATE Remove those training samples that are close to $\Dmeta$ (the similarity measure is the weighted cosine similarity defined in Equation \eqref{eq: weighted_cosine_sim})
     
    \STATE Repetitively run Algorithm \ref{alg: weighted_k_means} until there is no empty clusters, in which \revise{$\{\concatvec{j}\}_{j=1}^N$ is computed} 
    with \ourmethodone\ or \ourmethodtwo, and parameterized by $\{\Tilde{\Theta}_t\}_{t=1}^T$. Suppose this results in cluster centroids $\{\centroid_i\}_{i=1}^M$
     
    \STATE For each $\centroid_i (i=1,2,\dots, M)$, add the closest training sample to this centroid into $\Dmeta$
    
    \STATE Run meta re-weighting algorithm with $\Dmeta$ as the meta sample set, 
    resulting in a list of model parameters $\{\Theta_t\}_{t=1}^T$ at each iteration.
    
    \STATE $\{\Tilde{\Theta}_t\}_{t=1}^T := \{{\Theta}_t\}_{t=1}^T$
    
    \ENDWHILE

\caption{K-means clustering based meta sample selection}\label{alg: overall_algorithm}
\end{algorithmic}
\end{algorithm}

\begin{algorithm}
\textbf{Input:}{A set of model parameters $\{\Theta_t\}_{t=1}^T$, the epoch $t^*$ with best validation performance}\\
\textbf{Output:}{$\{\concatvec{j}\}_{j=1}^N$}
\begin{algorithmic}


\STATE Uniformly sample $K$ models from $\{\Theta_t\}_{t=t^*+1}^T$, resulting in $\{\Theta_{t_k}\}_{k=1}^K$ where $t^* < t_1 < t_2 < \dots < t_K$

\STATE Initialize $\{\concatvec{j}\}_{j=1}^N$ so that $\concatvec{j}=[]$ for all $j$.

\FOR{$k=1$ to $K$}

\IF{RBC}

\FOR{$j=1$ to $N$}

\STATE Append $\A_{j}(\Theta_{t_k})\Tilde{\x_j}(\Theta_{t_k})^\top$ to $\concatvec{j}$ according to Equation \eqref{eq: method_one_objective_function}

\ENDFOR
\ELSIF{GBC}
\STATE Compute $A^{(l)}(\Theta_{t_k})$ and $A(\Theta_{t_k})$ for all $l (l=1,2,\dots,L)$ as described in Section \nameref{sec: method_two}

\FOR{$j=1$ to $N$}
\STATE Conduct importance sampling from $[\nabla_{\Theta^{(1)}}f_j(\Theta_{t_k}) \dots \nabla_{\Theta^{(L)}}f_j(\Theta_{t_k})]$  
with sampling probability $[\frac{A^{(1)}(\Theta_{t_k})}{A(\Theta_{t_k})}, \frac{A^{(2)}(\Theta_{t_k})}{A(\Theta_{t_k})}, \dots, \frac{A^{(L)}(\Theta_{t_k})}{A(\Theta_{t_k})}]$, resulting in a list of gradients from $R$ layers: $\text{Grad}_{j} = [\nabla_{\Theta^{(l_1)}}f_j(\Theta_{t_k}), \nabla_{\Theta^{(l_2)}}f_j(\Theta_{t_k}), \dots,  \nabla_{\Theta^{(l_R)}}f_j(\Theta_{t_k})]$

\STATE Append $\text{Grad}_{j}$ to $\concatvec{j}$


\ENDFOR
\ENDIF

\ENDFOR
\FOR{$j=1$ to $N$}
\STATE Organize $\concatvec{j}$ as a block matrix
\ENDFOR

\caption{RBC and GBC}\label{alg: rbc_gbc}
\end{algorithmic}
\end{algorithm}

\subsubsection{Details of the adapted K-means algorithm}\label{appendix: adapted_k_means}

We tailor the vanilla K-means clustering algorithm to efficiently solve Equation \eqref{eq: weighted_k_means_obj_0}. Note that the K-means clustering algorithm is composed of two steps, i.e., the \textbf{assignment step} and the \textbf{update step}, which are conducted alternatively until convergence. In the \textbf{assignment step} of the modified K-means clustering algorithm, we follow the same principle of the vanilla K-means clustering algorithm. Specifically, we assign each training sample to its nearest cluster centroid, where the similarity between each training sample $j$ and each cluster centroid $i$ is the cosine similarity weighted by the norm of the centroid $\centroid_i$:
\begin{small}
\begin{align}\label{eq: weighted_cosine_sim}
    \begin{split}
        \|\centroid_i\|\cdot |\text{cosine}(\concatvec{j}, \centroid_i)|
    \end{split}
\end{align}
\end{small}

After each training sample $j$ is assigned to a certain cluster centroid $\centroid_i$, we proceed to update the cluster centroids in the \textbf{update step} given their assigned training samples. Indeed, according to Equation \eqref{eq: weighted_k_means_obj_0}, we conduct clustering on the normalized gradients, $\frac{\concatvec{j}}{\|\concatvec{j}\|}$, rather than $\concatvec{j}$ itself. Plus, since the overall similarity between each training sample $j$ and all clustering centroids is weighted by the norm of $\concatvec{j}$, we therefore update the cluster centroids by leveraging the following formula:

\begin{small}
\begin{align}\label{eq: weighted_k_means_update_step}
\begin{split}
    \centroid_i \leftarrow \sum\nolimits_{j \in \text{cluster }i} \frac{\|\concatvec{j}\|\cdot\frac{\concatvec{j}}{\|\concatvec{j}\|}}{\sum\nolimits_{j \in \text{cluster }i} \|\concatvec{j}\|} = \sum\nolimits_{j \in \text{cluster }i} \frac{\concatvec{j}}{\sum\nolimits_{j \in \text{cluster }i} \|\concatvec{j}\|},
\end{split}    
\end{align}
\end{small}

in which, the cluster centroid $\centroid_i$ is updated as the weighted mean of all the {\em normalized} samples that are assigned to this cluster.

\revise{The value of each $G_j$ is computed based on if RBC or GBC is used. The details are summarized in Algorithm \ref{alg: rbc_gbc}.}

In the end, we summarize this adapted K-means clustering algorithm in 
Algorithm \ref{alg: weighted_k_means}.

\subsubsection{Determining number of clusters and continuously adding meta samples}\label{appendix: cluster_count}
In this section, we further discuss how to determine the number of clusters and how to continuously add meta samples while the meta re-weighting algorithm is repetitively invoked.

First of all, we assume that the number of clusters, $M$, could be provided by the users. However, we observe that given an inappropriately large $M$, empty clusters are often generated, meaning that no samples are assigned to these clusters. Suppose that there are $M_{\text{empty}}$ empty clusters in total, we therefore restart the K-means clustering algorithm with the number of clusters as $M-M_{\text{cluster}}$. This process is repeated until there are no empty clusters. We then identify meta samples with the resulting clusters, which are used in the meta re-weighting algorithm.
Suppose we get model parameters ${\Theta}_t (t=1,2,\dots, T)$ from the meta re-weighting algorithm, we can also optionally run \ourmethodone\ or \ourmethodtwo\ again by leveraging ${\Theta}_t$ so that we can add more meta samples. This process is summarized in Algorithm \ref{alg: overall_algorithm}. Note that in the subsequent invocation of \ourmethodone\ or \ourmethodtwo, we remove a certain portion of training samples (e.g., half of them) that are closest to the existing meta samples and only cluster the remaining training samples to discover new meta samples. 




\section{Appendix: Mathematical details}

\subsection{Derivation of Equation \eqref{eq: gradient_inner_prod}}

First of all, we compute the partial gradient of Equation \eqref{eq: model_pre_update} with respect to $\w_j$, i.e.:
\begin{small}
\begin{align}\label{eq: model_preupdate_weight_grad}
\begin{split}
\frac{\partial \hat{\Theta}(\bw_t)}{\partial \w_j}& = \frac{\partial}{\partial \w_j}[\Theta_t - \alpha_t \cdot \frac{1}{N}\sum\nolimits_{r=1}^N\w_{r,t} \nabla_{\Theta} \f{r}{\Theta}|_{\Theta = \Theta_t}]\\
    & = - \frac{\alpha_t}{N} \nabla_{\Theta} \f{j}{\Theta}|_{\Theta = \Theta_t},
\end{split}
\end{align}
\end{small}

which utilizes the fact that except for $\w_{j,t} \nabla_{\theta} \f{j}{\Theta}$, all the other terms in Equation \eqref{eq: model_pre_update} do not depend on the weight $\w_{j,t}$. 



In addition, we utilize chain rule on the gradient of Equation \eqref{eq: sample_weights_update}, leading to:
\begin{small}
\begin{align}\label{eq: single_sample_weights_update_2}
\begin{split}
    & \nabla_{\w_j} \mf{i}{\hat{\Theta}(\bw_t)}= \langle\nabla_{\Theta}\mf{i}{\Theta}|_{\Theta = \hat{\Theta}(\bw_t)}, \frac{\partial \hat{\Theta}(\bw_t)}{\w_j}\rangle\\
    & = - \frac{\alpha_t}{N} \langle\nabla_{\Theta}\mf{i}{\Theta}|_{\Theta = \hat{\Theta}(\bw_t)}, \nabla_{\Theta} \f{j}{\Theta}|_{\Theta = \Theta_t}\rangle
\end{split}
\end{align}
\end{small}

\subsection{Derivation of Equation \eqref{eq: method_one_objective_function}}\label{appedix:derivation_method_one_objective}

In the main paper, we have used $\langle,\rangle$ to denote the \textit{Frobenius inner product} between matrices. But in the following analysis, the inner products between vectors will also appear, which are also conventionally denoted as $\langle,\rangle$. Therefore, to distinguish between these two types of inner products in what follows, we use $\langle,\rangle_F$ rather than $\langle,\rangle$ to represent the Frobenius inner product between matrices while $\langle,\rangle$ is used for representing the inner product between vectors.

First of all, suppose the loss function is cross-entropy loss, then we could have the following lemma for this loss:
\begin{lemma}\label{lemma_cl}
For cross-entropy loss, we can write it as the following form:
\begin{align}\label{eq: cross_entropy}
    L(\x, \y) = \log (\frac{\exp^{-x_j}}{\sum_{k=1}^m \exp^{-x_k}}),
\end{align}
in which we assume that $\y = j\ (j \in \{1,2,\dots, m\})$ and $\x = [x_1,x_2,\dots, x_m]$ is a vector of length $m$. 
Then the gradient of $L(\x,\y)$ with respect to the input $\x$ could be split into two parts, i.e. the label-dependent part and the label-free part.
\end{lemma}
\begin{proof}
The gradient of  $L(\x, \y)$ with respect to $\x$ could be derived as follows:
\begin{small}
\begin{align*}
    \frac{\partial L(\x, \y)}{\partial \x} = [\frac{\partial L(\x, \y)}{\partial x_1}, \frac{\partial L(\x, \y)}{\partial x_2}, \dots, \frac{\partial L(\x, \y)}{\partial x_r}, \dots, \frac{\partial L(\x, \y)}{\partial x_m}]^\top,
\end{align*}
\end{small}

in which,
\begin{small}
\begin{align*}
\frac{\partial L(\x, \y)}{\partial x_r} = \begin{cases} \frac{\exp^{-x_r}}{\sum_{k=1}^m \exp^{-x_k}}, r \neq y \\
\frac{\exp^{-x_r}}{\sum_{k=1}^m \exp^{-x_k}}-1, r = y
\end{cases}    
\end{align*}
\end{small}

As a consequence, $\frac{\partial L(\x, \y)}{\partial \x}$ could be written as:
\begin{small}
\begin{align}\label{eq: grad_decomposition}
    \begin{split}
        \frac{\partial L(\x, \y)}{\partial \x} & = [\frac{\partial L(\x, \y)}{\partial x_1}, \frac{\partial L(\x, \y)}{\partial x_2}, \dots, \frac{\partial L(\x, \y)}{\partial x_r}, \dots, \frac{\partial L(\x, \y)}{\partial x_m}]^\top\\
        & = [\frac{\exp^{-x_1}}{\sum_{k=1}^m \exp^{-x_k}}, \frac{\exp^{-x_2}}{\sum_{k=1}^m \exp^{-x_k}}, \dots, \frac{\exp^{-x_r}}{\sum_{k=1}^m \exp^{-x_k}}, \dots, \frac{\exp^{-x_m}}{\sum_{k=1}^m \exp^{-x_k}}]^\top - \text{onehot}(j)\\
        & = \underbrace{\text{softmax}(\x)}_{\textit{label free part}} - \underbrace{\text{onehot}(y)}_{\textit{label dependent part}}
    \end{split}
\end{align}
\end{small}


\end{proof}

Note that when we only consider the gradient of the last layer, whose parameters are denoted as $\Theta^{(-1)}$, $\f{j}{\Theta}$ and $\mf{i}{\Theta}$ could be represented as:
\begin{small}
\begin{align*}
    \f{j}{\Theta} = L(\lastxout{j}(\Theta),\y_j) = L(\Theta^{(-1)}\Tilde{\x_j}(\Theta), \y_j),
    \mf{i}{\Theta} = L(\lastxout{\text{meta},i}(\Theta),\my{i}) = L(\Theta^{(-1)}\Tilde{\mx{i}}(\Theta), \my{i}). 
\end{align*}
\end{small}

Recall that in Section \nameref{sec: method_one} and Section \nameref{sec: application_label_noise}, we use $\lastxout{j}(\Theta)$ and $\Tilde{\x_j}(\Theta)$ to denote the input of the last linear layer and the input of the softmax layer (i.e. the output of the last linear layer) given the training sample $j$. Similarly, $\lastxout{\text{meta},i}(\Theta)$ and $\Tilde{\mx{i}}(\Theta)$ represent the input and the output of the last linear layer given the meta sample $i$.

Then we could derive the gradient of $\f{j}{\Theta}$ (same for $\mf{i}{\Theta}$) with respect to the vectorized $\Theta^{(-1)}$ by leveraging the chain rule, leading to:
\begin{small}
\begin{align}\label{eq: vectorized_inner_prod_last_layer}
    \nabla_{\vect{\Theta^{(-1)}}} \f{j}{\Theta} = \frac{\partial \lastxout{j}(\Theta)}{\partial \vect{\Theta^{(-1)}}}\cdot \frac{\partial L(\lastxout{j}(\Theta), \Tilde{\y_j})}{\partial \lastxout{j}(\Theta)}.
\end{align}
\end{small}

By leveraging Lemma \ref{lemma_cl}, the above formula could be rewritten as:
\begin{small}
\begin{align*}
    \nabla_{\vect{\Theta^{(-1)}}} \f{j}{\Theta} = \frac{\partial \lastxout{j}(\Theta)}{\partial \vect{\Theta^{(-1)}}}\cdot  [\text{softmax}(\Theta^{(-1)}\Tilde{\x_j}(\Theta)) - \text{onehot}(\y_j)]
\end{align*}
\end{small}

and for $\frac{\partial \lastxout{j}(\Theta)}{\partial \vect{\Theta^{(-1)}}}$, it could be further derived as follows by utilizing the calculus on block matrix:
\begin{small}
\begin{align*}
    \frac{\partial \lastxout{j}(\Theta)}{\partial \vect{\Theta^{(-1)}}} =\frac{\partial \Theta^{(-1)}\Tilde{\x_j}(\Theta)}{\partial \vect{\Theta^{(-1)}}} = [\Tilde{\x_j}(\Theta) \otimes \textbf{I}]^\top = [\Tilde{\x_{j,1}}(\Theta)\textbf{I}, \Tilde{\x_{j,2}}(\Theta)\textbf{I}, \dots \Tilde{\x_{j,m}}(\Theta)\textbf{I}]^\top
\end{align*}
\end{small}

in which $\x_{j,r}(\Theta)$ denotes the $r_{th}$ entry of the vector $\x_{j}(\Theta)$ and $\otimes$ denotes the Kronecker product \citesupp{henderson1983history} on two matrices.

We then plug the above formula into Equation \eqref{eq: vectorized_inner_prod_last_layer}, resulting in:
\begin{small}
\begin{align*}
    \nabla_{\vect{\Theta^{(-1)}}} \f{j}{\Theta}& = [\Tilde{\x_{j,1}}(\Theta)\textbf{I}, \Tilde{\x_{j,2}}(\Theta)\textbf{I}, \dots \Tilde{\x_{j,m}}(\Theta)\textbf{I}]^\top \cdot [\text{softmax}(\Theta^{(-1)}\Tilde{\x_j}(\Theta)) - \text{onehot}(\y_j)]\\
    & = \vect{\left[\text{softmax}(\Theta^{(-1)}\Tilde{\x_j}(\Theta)) - \text{onehot}(\y_j)\right] \Tilde{\x_{j}}(\Theta)^\top},
\end{align*}
\end{small}

Recall that in Section \nameref{sec: method_one}, we use $\A_j(\Theta)$ to denote $\text{softmax}(\Theta^{(-1)}\Tilde{\x_j}(\Theta)) - \text{onehot}(\y_j)$. Therefore, the above formula could be rewritten as:
\begin{small}
\begin{align*}
\begin{split}
    \nabla_{\vect{\Theta^{(-1)}}} \f{j}{\Theta} = \vect{\A_j(\Theta)\Tilde{\x_{j}}(\Theta)^\top}
\end{split}
\end{align*}
\end{small}

Similarly, the following equation holds for the meta sample $i$:
\begin{small}
\begin{align*}
    \begin{split}
        \nabla_{\vect{\Theta^{(-1)}}} \mf{i}{\Theta} = \vect{\A_{\text{meta}, i}(\Theta)\Tilde{\mx{i}}(\Theta)^\top}
    \end{split}
\end{align*}
\end{small}

As a result, we can compute the inner product between $\nabla_{\vect{\Theta^{(-1)}}} \mf{i}{\Theta}$ and $\nabla_{\vect{\Theta^{(-1)}}} \f{j}{\Theta}$ by leveraging the above two formulas, leading to:
\begin{small}
\begin{align*}
    \begin{split}
        \langle \nabla_{\vect{\Theta^{(-1)}}} \mf{i}{\Theta}, \nabla_{\vect{\Theta^{(-1)}}} \f{j}{\Theta} \rangle & = \langle \vect{\A_j(\Theta)\Tilde{\x_{j}}(\Theta)^\top}, \vect{\A_{\text{meta}, i}(\Theta)\Tilde{\mx{i}}(\Theta)^\top} \rangle\\
        & = \langle \A_j(\Theta)\Tilde{\x_{j}}(\Theta)^\top, \A_{\text{meta}, i}(\Theta)\Tilde{\mx{i}}(\Theta)^\top \rangle_F
    \end{split}
\end{align*}
\end{small}

We can then plug the above formula into Equation \eqref{eq: our_init_objective}, i.e. MSSO, leading to:
\begin{small}
\begin{align*}
    \text{MSSO} & = \max_{\Dmeta}\sum\nolimits_{j=1}^N \left|\sum\nolimits_{i=1}^M \langle\concatvec{j}, \concatvec{\text{meta}, i}\rangle\right|\\
    & = \max_{\Dmeta}\sum\nolimits_{j=1}^N \left|\sum\nolimits_{i=1}^M \sum\nolimits_{\Theta_t}\langle\nabla_{\Theta} \f{j}{\Theta}|_{\Theta=\Theta_t}, \nabla_{\Theta} \mf{i}{\Theta}|_{\Theta=\Theta_t}\rangle\right|\\
    & = \max_{\Dmeta}\sum\nolimits_{j=1}^N \left|\sum\nolimits_{i=1}^M \sum\nolimits_{\Theta_t}\langle\nabla_{\vect{\Theta}} \f{j}{\Theta}|_{\Theta=\Theta_t}, \nabla_{\vect{\Theta}} \mf{i}{\Theta}|_{\Theta=\Theta_t}\rangle\right|\\
    & \approx \max_{\Dmeta}\sum\nolimits_{j=1}^N \left|\sum\nolimits_{i=1}^M \sum\nolimits_{\Theta_t}\langle\nabla_{\vect{\Theta^{(-1)}}} \f{j}{\Theta}|_{\Theta=\Theta_t}, \nabla_{\vect{\Theta^{(-1)}}} \mf{i}{\Theta}|_{\Theta=\Theta_t}\rangle\right|\\
    & = \max_{\Dmeta}\sum\nolimits_{j=1}^N \left|\sum\nolimits_{i=1}^M \sum\nolimits_{\Theta_t}\langle\A_{j}(\Theta_t)\Tilde{\x_j}(\Theta_t)^\top, \A_{\text{meta}, i}(\Theta_t)\Tilde{\mx{i}}(\Theta_t)^\top\rangle_F\right|
\end{align*}
\end{small}

This thus concludes the derivation of Equation \eqref{eq: method_one_objective_function}.

\subsection{Proof of Theorem \ref{theorem: main}}\label{sec: theorem_main}
\begin{proof}
By utilizing the following property concerning the absolute values of the sum of two numbers \citesupp{stewart2011calculus}:
\begin{small}
\begin{align*}
    |a+b| = ||a| - |b||, \text{if } a\cdot b \leq 0,
\end{align*}
\end{small}



the inner most sum in Equation \eqref{eq: our_init_objective} could be rewritten as follows:
\begin{small}
\begin{align*}
    & |\sum\nolimits_{i=1}^M \langle\concatvec{j}, \centroid_i\rangle|\\
    & = |\sum\nolimits_{\langle\concatvec{j}, \centroid_i\rangle > 0} \langle\concatvec{j}, \centroid_i\rangle + \sum\nolimits_{\langle\concatvec{j}, \centroid_i\rangle < 0} \langle\concatvec{j}, \centroid_i\rangle|\\
    & = ||\sum\nolimits_{\langle\concatvec{j}, \centroid_i\rangle > 0} \langle\concatvec{j}, \centroid_i\rangle| - |\sum\nolimits_{\langle\concatvec{j}, \centroid_i\rangle < 0} \langle\concatvec{j}, \centroid_i\rangle||
\end{align*}
\end{small}

Also, by utilizing the following property concerning the sum of the absolute value of two numbers \citesupp{stewart2011calculus}:
\begin{small}
\begin{align*}
    |a|+|b| = |a+b|, \text{if } a\cdot b \geq 0
\end{align*}
\end{small}

for the innermost sum in Equation \eqref{eq: transformed_k_means_objective} could be rewritten as follows:
\begin{small}
\begin{align*}
    & \sum\nolimits_{i=1}^M |\langle\concatvec{j}, \centroid_i\rangle|\\
    & = \sum\nolimits_{\langle\concatvec{j}, \centroid_i\rangle > 0}| \langle\concatvec{j}, \centroid_i\rangle| + \sum_{\langle\concatvec{j}, \centroid_i\rangle < 0} |\langle\concatvec{j}, \centroid_i\rangle|\\
    & = |\sum\nolimits_{\langle\concatvec{j}, \centroid_i\rangle > 0} \langle\concatvec{j}, \centroid_i\rangle| + |\sum\nolimits_{\langle\concatvec{j}, \centroid_i\rangle < 0}\langle\concatvec{j}, \centroid_i\rangle|
\end{align*}
\end{small}

Then we compute the ratio between the above two formulas, leading to:
\begin{small}
\begin{align}\label{eq: theorem_bound_derivation_1}
\begin{split}
    &\frac{|\sum\nolimits_{i=1}^M \langle\concatvec{j}, \centroid_i\rangle|}{\sum\nolimits_{i=1}^M |\langle\concatvec{j}, \centroid_i\rangle|} = \frac{||\sum\nolimits_{\langle\concatvec{j}, \centroid_i\rangle > 0} \langle\concatvec{j}, \centroid_i\rangle| - |\sum\nolimits_{\langle\concatvec{j}, \centroid_i\rangle < 0} \langle\concatvec{j}, \centroid_i\rangle||}{|\sum\nolimits_{\langle\concatvec{j}, \centroid_i\rangle > 0} \langle\concatvec{j}, \centroid_i\rangle| + |\sum\nolimits_{\langle\concatvec{j}, \centroid_i\rangle < 0} \langle\concatvec{j}, \centroid_i\rangle|}\\
    & = \begin{cases}
    \frac{\frac{|\sum\nolimits_{\langle\concatvec{j}, \centroid_i\rangle < 0} \langle\concatvec{j}, \centroid_i\rangle|}{|\sum\nolimits_{\langle\concatvec{j}, \centroid_i\rangle > 0} \langle\concatvec{j}, \centroid_i\rangle|}-1}{\frac{|\sum\nolimits_{\langle\concatvec{j}, \centroid_i\rangle < 0} \langle\concatvec{j}, \centroid_i\rangle|}{|\sum\nolimits_{\langle\concatvec{j}, \centroid_i\rangle > 0} \langle\concatvec{j}, \centroid_i\rangle|}+1} = 1 - \frac{2}{\frac{|\sum\nolimits_{\langle\concatvec{j}, \centroid_i\rangle < 0} \langle\concatvec{j}, \centroid_i\rangle|}{|\sum\nolimits_{\langle\concatvec{j}, \centroid_i\rangle > 0} \langle\concatvec{j}, \centroid_i\rangle|}+1} & \text{if } |\sum\nolimits_{\langle\concatvec{j}, \centroid_i\rangle < 0} \langle\concatvec{j}, \centroid_i\rangle| > |\sum\nolimits_{\langle\concatvec{j}, \centroid_i\rangle > 0} \langle\concatvec{j}, \centroid_i\rangle|\\
    \frac{\frac{|\sum\nolimits_{\langle\concatvec{j}, \centroid_i\rangle > 0} \langle\concatvec{j}, \centroid_i\rangle|}{|\sum\nolimits_{\langle\concatvec{j}, \centroid_i\rangle < 0} \langle\concatvec{j}, \centroid_i\rangle|}-1}{\frac{|\sum\nolimits_{\langle\concatvec{j}, \centroid_i\rangle > 0} \langle\concatvec{j}, \centroid_i\rangle|}{|\sum\nolimits_{\langle\concatvec{j}, \centroid_i\rangle < 0} \langle\concatvec{j}, \centroid_i\rangle|}+1} = 1 - \frac{2}{\frac{|\sum\nolimits_{\langle\concatvec{j}, \centroid_i\rangle > 0} \langle\concatvec{j}, \centroid_i\rangle|}{|\sum\nolimits_{\langle\concatvec{j}, \centroid_i\rangle < 0} \langle\concatvec{j}, \centroid_i\rangle|}+1} & \text{if } |\sum\nolimits_{\langle\concatvec{j}, \centroid_i\rangle > 0} \langle\concatvec{j}, \centroid_i\rangle| > |\sum\nolimits_{\langle\concatvec{j}, \centroid_i\rangle < 0} \langle\concatvec{j}, \centroid_i\rangle|
    \end{cases}
\end{split}
\end{align}
\end{small}

Then by leveraging the assumption of this Theorem, we know that:
\begin{small}
\begin{align*}
\begin{cases}
    \frac{|\sum\nolimits_{\langle\concatvec{j}, \centroid_i\rangle > 0} \langle\concatvec{j}, \centroid_i\rangle|}{|\sum\nolimits_{\langle\concatvec{j}, \centroid_i\rangle < 0} \langle\concatvec{j}, \centroid_i\rangle|} > D & \text{ if } |\sum\nolimits_{\langle\concatvec{j}, \centroid_i\rangle < 0} \langle\concatvec{j}, \centroid_i\rangle| > |\sum\nolimits_{\langle\concatvec{j}, \centroid_i\rangle > 0} \langle\concatvec{j}, \centroid_i\rangle|\\
    \frac{|\sum\nolimits_{\langle\concatvec{j}, \centroid_i\rangle < 0} \langle\concatvec{j}, \centroid_i\rangle|}{|\sum\nolimits_{\langle\concatvec{j}, \centroid_i\rangle > 0} \langle\concatvec{j}, \centroid_i\rangle|} > D & \text{ if } |\sum\nolimits_{\langle\concatvec{j}, \centroid_i\rangle > 0} \langle\concatvec{j}, \centroid_i\rangle| > |\sum\nolimits_{\langle\concatvec{j}, \centroid_i\rangle < 0} \langle\concatvec{j}, \centroid_i\rangle|
\end{cases}
\end{align*}
\end{small}

Therefore, Equation \eqref{eq: theorem_bound_derivation_1} could be bounded as follows:
\begin{small}
\begin{align*}
    \text{Equation }\eqref{eq: theorem_bound_derivation_1} & = \begin{cases}
    1 - \frac{2}{\frac{|\sum\nolimits_{\langle\concatvec{j}, \centroid_i\rangle < 0} \langle\concatvec{j}, \centroid_i\rangle|}{|\sum\nolimits_{\langle\concatvec{j}, \centroid_i\rangle > 0} \langle\concatvec{j}, \centroid_i\rangle|}+1} & \text{if } |\sum\nolimits_{\langle\concatvec{j}, \centroid_i\rangle < 0} \langle\concatvec{j}, \centroid_i\rangle| > |\sum\nolimits_{\langle\concatvec{j}, \centroid_i\rangle > 0} \langle\concatvec{j}, \centroid_i\rangle|\\
    1 - \frac{2}{\frac{|\sum\nolimits_{\langle\concatvec{j}, \centroid_i\rangle > 0} \langle\concatvec{j}, \centroid_i\rangle|}{|\sum\nolimits_{\langle\concatvec{j}, \centroid_i\rangle < 0} \langle\concatvec{j}, \centroid_i\rangle|}+1} & \text{if } |\sum\nolimits_{\langle\concatvec{j}, \centroid_i\rangle > 0} \langle\concatvec{j}, \centroid_i\rangle| > |\sum\nolimits_{\langle\concatvec{j}, \centroid_i\rangle < 0} \langle\concatvec{j}, \centroid_i\rangle|
    \end{cases}\\
    & > \begin{cases}
    1 - \frac{2}{D+1} & \text{if } |\sum\nolimits_{\langle\concatvec{j}, \centroid_i\rangle < 0} \langle\concatvec{j}, \centroid_i\rangle| > |\sum\nolimits_{\langle\concatvec{j}, \centroid_i\rangle > 0} \langle\concatvec{j}, \centroid_i\rangle|\\
    1 - \frac{2}{D+1} & \text{if } |\sum\nolimits_{\langle\concatvec{j}, \centroid_i\rangle > 0} \langle\concatvec{j}, \centroid_i\rangle| > |\sum\nolimits_{\langle\concatvec{j}, \centroid_i\rangle < 0} \langle\concatvec{j}, \centroid_i\rangle|
    \end{cases}\\
    & = \frac{D-1}{D+1}
\end{align*}
\end{small}

In Section \nameref{sec: supple_exp}, we will calculate the value of $D$ empirically. Recall that the value of $D$ needs to be significantly larger than 1 so that solving MCO (i.e. Equation \eqref{eq: transformed_k_means_objective}) could end up with a reasonable approximation to the solution of MSSO (i.e. Equation \eqref{eq: our_init_objective}), which could be empirically justified on \mnist\ and \cifar\ dataset.

\end{proof}

\subsection{Generalization of our methods for \citesupp{ren2018learning}}\label{appendix: generalization_method}
\revise{As discussed in Section \nameref{sec: background}, our methods mainly utilize the meta-reweighting algorithm proposed in \citesupp{shu2019meta}. However, we can show that our methods can also support other meta-reweighting algorithms, such as \citesupp{ren2018learning}, which we illustrate in this section.}


\revise{First of all, we notice that \citesupp{ren2018learning} and \citesupp{shu2019meta} mainly differ in how to update the sample weights at each time step. Specifically, we reformulate Equation \eqref{eq: model_pre_update}-Equation \eqref{eq: model_post_update} according to \citesupp{ren2018learning} as follows:}
\begin{small}
\begin{empheq}[box=\widefbox]{align}
    & \textbf{Meta re-weighting \citesupp{ren2018learning}:} \nonumber \\
    &\hat{\Theta}(\bw_t) = \Theta_t - \alpha_t \cdot \frac{1}{N}\sum\nolimits_{j=1}^N\w_{j,t} \nabla_{\Theta} \f{j}{\Theta}|_{\Theta = \Theta_t, \w_{j,t}=0}
    \label{eq: model_pre_update_general}\\
    & \hat{\w}_{j,t+1} =  - \eta_t \cdot \frac{1}{M}\sum\nolimits_{i=1}^M \nabla_{\w_{j,t}} \mf{i}{\hat{\Theta}(\bw_t)} |_{\w_{j,t}=0}
    \label{eq: sample_weights_update_general}\\
    & \hat{\w}_{j,t+1} = \max\{\hat{\w}_{j,t+1}, 0\}, \w_{j,t+1} = \text{normalize}([\hat{\w}_{1,t+1}, \hat{\w}_{2,t+1},\dots,\hat{\w}_{N,t+1}]) \label{eq: sample_weights_update_general2}\\
    & \Theta_{t+1} = \Theta_t - \alpha_t\cdot \frac{1}{N}\sum\nolimits_{j=1}^N \w_{j,t+1}\nabla_{\Theta} \f{j}{\Theta}|_{\Theta = \Theta_t}
    \label{eq: model_post_update_general},
\end{empheq}
\end{small}

\revise{Recall that the intuition of our methods is to find meta samples so that each sample weight $\w_j$ is effectively updated during the training process instead of staying close to its random initialization. Therefore, we hope to maximize the following cumulative gradients across all the samples during the entire training process, i.e.:}
\begin{small}
\begin{align}\label{eq: meta_reweighting_goal_general}
    \max_{\Dmeta}\sum\nolimits_{\hat{\Theta}(\bw_t)} \frac{1}{M}\sum\nolimits_{i=1}^M \nabla_{\w_j} \mf{i}{\hat{\Theta}(\bw_t)}, \text{for all $j=(1,2,\dots,N)$},
\end{align}
\end{small}

\revise{Note that different from Equation \eqref{eq: meta_reweighting_goal}, no absolute value is calculated in Equation \eqref{eq: meta_reweighting_goal_general}. Then by following the same derivation as Equation \eqref{eq: gradient_inner_prod}, we can get the following formula which is similar to Equation \eqref{eq: our_init_objective0}:}
\begin{small}
\begin{align}\label{eq: our_init_objective0_general}
\begin{split}
    & \max_{\Dmeta}\sum\nolimits_{j=1}^N{\sum\nolimits_{\hat{\Theta}(\bw_t), \Theta_t}\sum\nolimits_{i=1}^M \langle\nabla_{\Theta} \mf{i}{\Theta}|_{\Theta = \hat{\Theta}(\bw_t)}, \nabla_{\Theta} \f{j}{\Theta}|_{\Theta = \Theta_t}\rangle}\\
    & \approx \max_{\Dmeta}\sum\nolimits_{j=1}^N{\sum\nolimits_{i=1}^M \sum\nolimits_{\Theta_t}\langle\nabla_{\Theta} \mf{i}{\Theta}|_{\Theta =\Theta_t}, \nabla_{\Theta} \f{j}{\Theta}|_{\Theta = \Theta_t}\rangle}.\\
\end{split}
\end{align}
\end{small}

\revise{Again, the only difference between Equation \eqref{eq: our_init_objective0} and Equation \eqref{eq: our_init_objective0_general} is the existence of the absolute value computation. As a result, when the meta-reweighting algorithm from \citesupp{ren2018learning} is used, our methods, including both \ourmethodone\ and \ourmethodtwo, can be employed for selecting meta samples, except that the absolute value is not evaluated. }

\subsection{Analysis of the gradient with and without label-free part}\label{appendix: label_free_gradient}

According to Section \nameref{sec: application_label_noise}, in the presence of label noises, the sample-wise gradient $\nabla_{\Theta} \f{j}{\Theta}$ could be split into {\em label-free part} and the {\em label-dependent part}, which are represented as follows:
\begin{small}
\begin{align}\label{eq: split_gradient}
\begin{split}
    & \nabla_{\Theta} \f{j}{\Theta} = [\underbrace{\frac{\partial \lastxout{j}(\Theta)}{\partial \Theta}\cdot \text{softmax}(\lastxout{j}(\Theta))}_{\textit{label-free gradient}} - \underbrace{\frac{\partial \lastxout{j}(\Theta)}{\partial \Theta} \cdot \text{onehot}(\y)}_{\textit{label-dependent gradient}}],
\end{split}
\end{align}
\end{small}
in which $\lastxout{j}(\Theta)$ represents the input of the softmax layer of the neural model. The derivation of Equation \eqref{eq: split_gradient} is discussed in Lemma \ref{lemma_cl} in Appendix \nameref{appedix:derivation_method_one_objective}.

Recall that in Section \nameref{sec: application_label_noise}, we only focus on the label-free part in the sample-wise gradient, i.e., $\nabla_{\Theta} \widetilde{\f{j}{\Theta}}$. Therefore, by replacing $\nabla_{\Theta}\f{j}{\Theta}$ with $\nabla_{\Theta} \widetilde{\f{j}{\Theta}}$ and $\nabla_{\Theta} \mf{i}{\Theta}$ with $\nabla_{\Theta} \widetilde{\mf{i}{\Theta}}$ in Equation \eqref{eq: transformed_k_means_objective}, the objective function becomes:
\begin{small}
\begin{align}\label{eq: transformed_k_means_objective_label_free}
&\max_{\Dmeta}(\sum_{j=1}^N \sum_{i=1}^M |\langle \widetilde{\concatvec{j}}, \widetilde{\concatvec{\text{meta}, i}}\rangle |),
\end{align}
\end{small}
in which, 
\begin{small}
\begin{align*}
\widetilde{\concatvec{j}} = \left[ \nabla_{\Theta} \widetilde{\f{j}{\Theta}}|_{\Theta = \Theta_1} \cdots \nabla_{\Theta}\widetilde{\f{j}{\Theta}}|_{\Theta = \Theta_t} \cdots \right], \widetilde{\concatvec{\text{meta},i}} = \left[ \nabla_{\Theta} \widetilde{\mf{i}{\Theta}}|_{\Theta = \Theta_1} \cdots \nabla_{\Theta} \widetilde{\mf{i}{\Theta}}|_{\Theta = \Theta_t} \cdots \right]
\end{align*}
\end{small}

In what follows, i.e., Theorem \ref{theorem: approximate_label_free}, we are going to show that with some assumptions, solving Equation \eqref{eq: transformed_k_means_objective_label_free} can produce approximately the same clusters as the ones produced by solving Equation \eqref{eq: transformed_k_means_objective} where the sample labels are involved. 

\begin{theorem}\label{theorem: approximate_label_free}
Suppose solving Equation \eqref{eq: method_one_objective_function} with the weighted K-means algorithm, i.e., Algorithm \ref{alg: weighted_k_means}, assigns the training sample $j$ to the cluster centroid $c^{(j)}$. Then if the following assumptions hold for the training sample $i$:
\begin{enumerate}
    \item The similarity between the training sample $j$ and the assigned centroid $c^{(j)}$, 
    is greater than $\alpha_j$ and the similarity between each training sample $j$ and any other cluster centroid $i$ is smaller than $\beta_j$ ($\alpha_j > \beta_j$), i.e.:
    \begin{small}
    \begin{align*}
    &|\langle \widetilde{\concatvec{j}}, \widetilde{\concatvec{\text{meta}, c^{(j)}}}\rangle| \geq \alpha_j\\
    & |\langle \widetilde{\concatvec{j}}, \widetilde{\concatvec{\text{meta}, i}}\rangle| \leq \beta_j, i \neq c^{(j)}
    \end{align*}
    \end{small}
    \label{assp: bound}
    \item The ratio between $|\langle \widetilde{\concatvec{j}}, \widetilde{\concatvec{\text{meta}, i}}\rangle|$ and $|\langle \concatvec{j}, \concatvec{\text{meta},i}\rangle|$ is lower bounded by $L_j$ and upper bounded by $U_j$ (The values of $L_j$ and $U_j$ could depend on each training samples $j$). 
\end{enumerate}

then the following inequalities on $|\langle \concatvec{j}, \concatvec{\text{meta}, i}\rangle|$ hold:
\begin{small}
\begin{align*}
&|\langle \concatvec{j}, \concatvec{\text{meta}, c^{(j)}}\rangle| \geq L_j \alpha_j\\
& |\langle \concatvec{j}, \concatvec{\text{meta}, i}\rangle| \leq U_j \beta_j, i \neq c^{(j)},
\end{align*}
in which $\concatvec{\text{meta}, i}$ involves the ground-truth labels of the meta sample $i$. 
\end{small}

\end{theorem}
This theorem thus suggests that if $L_j\alpha_j$ is greater than $U_j\beta_j$, then the cluster centroid $c^{j}$ could be still the closest centroid of the training sample $j$. In Appendix \nameref{appendix: quantitative_res}, we will empirically 
count how many training samples can still be assigned to the same cluster centroid after we change the similarity measure from $|\langle \widetilde{\concatvec{j}}, \widetilde{\concatvec{\text{meta}, i}}\rangle|$ to $|\langle \concatvec{j}, \concatvec{\text{meta}, i}\rangle|$.

\section{Supplemental experiments}\label{sec: supple_exp}
In this section, we provide some extra experiments which could not be included in the main paper. 

\subsection{Details of all hyper-parameters in the experiments}
To run meta re-weighting algorithm, we use SGD with initial learning rate 0.1, momentum 0.8, and weight decay $5\times10^{-4}$ for the \cifar\ experiments, and we use SGD with constant learning rate 0.1 for the \mnist\ experiments. We use a mini-batch size of 4096 and 128 for \mnist\ and \cifar\ respectively.
Following \citesupp{shu2019meta},  we also use a learning rate decay for \cifar\ experiments such that the learning rate is divided by 10 at epoch 80 and epoch 90 (100 epochs in total). 

As indicated in Section \nameref{sec: chicken_egg_problem}, for both \ourmethodone\ and \ourmethodtwo, we need to sample certain model parameters from each epoch of the warm-up phase, i.e. the meta re-weighting process with some randomly selected samples as the meta samples. We therefore sample those model parameters every 20 epochs after the epoch $t^*$ where the best model parameters occur. Plus, for \ourmethodtwo, as mentioned in Section \nameref{sec: method_two}, we sample several the model parameters at the granularity of network layers to estimate Equation \eqref{eq: inner_prod_by_layer_by_importance_score} and the number of the sampled network layers is set as 5. After collecting gradients of each training sample for \ourmethodone\ and \ourmethodtwo, we then run the weighted K-means clustering algorithm (i.e. Algorithm \ref{alg: weighted_k_means}) long enough. To guarantee the convergence, the number of epochs is set as 200.

\subsection{More quantitative results}\label{appendix: quantitative_res}

\subsubsection{Empirical evaluations of Theorem \ref{theorem: main}} Note that Theorem \ref{theorem: main} depends on the assumption that the positive part or the negative part in the innermost sum of Equation \eqref{eq: our_init_objective} are dominant over the negative terms or vice versa. This assumption is not theoretically analyzed, which is thus verified empirically in this section. Specifically, we calculate the ratio between the (dominant) positive part and the negative part (or the dominant negative part with respect to the positive part), (denoted as $D_j$) for each training sample $j$ in the label noise experiments when \ourmethodone\ is used. Note that the value of $D$ in Theorem \ref{theorem: main} equals to the minimum of all $D_j$. We therefore report the statistics of $D_j$
in Table \ref{Table: D_value}.

\begin{table}[ht]
\caption{The statistics of the value of $D_j (j=1,2,\dots,)$ in label noise experiments (\ourmethodone)}\label{Table: D_value}
\small
\centering
\begin{tabular}[!h]{>{\arraybackslash}p{2cm}|>{\centering\arraybackslash}p{1.8cm}>{\centering\arraybackslash}p{1.8cm}|>{\centering\arraybackslash}p{1.8cm}>{\centering\arraybackslash}p{1.8cm}} \hline\hline
 Dataset &\multicolumn{2}{c|}{\cifarone}&\multicolumn{2}{c}{\cifartwo} \\ \hline
 Noise type & uniform & adversarial & uniform & adversarial \\ \hline
minimum & 1.70 & inf & 1.01 & 1.01 \\ \hline
5\%-quantile & 4.38 & inf & 1.40 &1.68 \\ \hline \hline
\end{tabular}
\end{table}

As indicated in Table \ref{Table: D_value}, for both \cifarone\ with different types of noises, the values of $D$ are all significantly greater than 1 across all training samples, thus verifying the assumption of Theorem \ref{theorem: main} (inf means that all the terms in the innermost sum of Equation \eqref{eq: our_init_objective} have the same signs). Similar results are also observed in the class-imbalance experiments. For \cifartwo, although the minimum value of $D_j$ is almost 1, there are less than 5\% training samples with near-one $D_j$ value. Therefore, after removing this small portion of outlier samples, the assumption of Theorem \eqref{theorem: main} still holds.

\subsubsection{Empirical evaluations of Theorem \ref{theorem: approximate_label_free}}

Recall that due to the unavailable ground-truth labels when \ourmethodone\ or \ourmethodtwo\ are used in the presence of label noises, we proposed to only employ the label-free part in the sample-wise gradients as the input to our methods, which results in a label-free similarity measure, $|\langle \widetilde{\concatvec{j}}, \widetilde{\concatvec{\text{meta}, i}}\rangle|$, and a label-free objective function in Equation \eqref{eq: transformed_k_means_objective_label_free}. Theorem \ref{theorem: approximate_label_free} explored under what conditions, by using the label-aware similarity, $|\langle {\concatvec{j}}, {\concatvec{\text{meta}, i}}\rangle|$, the sample $j$ could still be closest to the cluster centroid, $c^{(j)}$, which is determined by solving the label-free objective function, Equation \eqref{eq: transformed_k_means_objective_label_free}. This type of samples are named as {\em stable samples} and we count how many such samples exist in the label noise experiments, which is reported in Table \ref{Table: label_free_approx_eval}.

\begin{table}[ht]
\caption{The number of stable samples (out of 25k) in the label noise experiments (\ourmethodone)}\label{Table: label_free_approx_eval}
\small
\centering
\begin{tabular}[!h]{>{\arraybackslash}p{2cm}|>{\centering\arraybackslash}p{1.8cm}>{\centering\arraybackslash}p{1.8cm}|>{\centering\arraybackslash}p{1.8cm}>{\centering\arraybackslash}p{1.8cm}} \hline\hline
 Dataset &\multicolumn{2}{c|}{\cifarone}&\multicolumn{2}{c}{\cifartwo} \\ \hline
 Noise type & uniform & adversarial & uniform & adversarial \\ \hline
count & 23593 & 20517 & 20543 & 20547 \\ \hline
\end{tabular}
\end{table}


Note that due to the existence of the meta samples from the warm-up phase, according to Algorithm \ref{alg: overall_algorithm}, we only conduct the weighted K-means clustering algorithm on the samples that are far away from the existing meta samples. As mentioned in Appendix \nameref{appendix: cluster_count}, the number of such samples is around half of the entire training set, i.e., around 25k for \cifarone\ and \cifartwo\ dataset. As indicated in Table \ref{Table: label_free_approx_eval}, over 80\% of the training samples are {\em stable samples} for both \cifarone\ and \cifartwo\ dataset. This thus suggests that clustering the label-free part of the sample-wise gradients in our methods could lead to a reasonable approximation of the results produced by clustering the ground-truth-label-aware gradients.

\begin{table}[ht]
\caption{Test accuracy on \cifarone\ with noise rate 60\% by repetitively adding meta samples}\label{table: repeat_run}
\small
\centering
\begin{tabular}[!h]{>{\arraybackslash}p{3cm}|>{\centering\arraybackslash}p{1.5cm}>{\centering\arraybackslash}p{1.5cm}>{\centering\arraybackslash}p{1.5cm}>{\centering\arraybackslash}p{1.5cm}} \hline\hline
Meta sample count & 20 & 30 & 40 & 50 \\ \hline
\certain & 66.15 & \underline{81.03}&81.65 & 81.57 \\
TA-VAAL &66.73 &68.63 & 60.63& 61.47\\ \hline
\ourmethodone & \textbf{78.95} & \textbf{82.10}& \textbf{83.36}& \underline{81.76}\\
\ourmethodtwo & \underline{76.46}& 80.58&\underline{81.96} & \textbf{81.93}\\ \hline\hline
\end{tabular}
\end{table}

\paragraph{Results with real label noise} Note that so far we only studied the performance of our methods by polluting the labels of the benchmark datasets in a synthetic manner, which may not occur in the real applications. We therefore follow \citesupp{wei2021learning} to add real human labeling errors to \cifarone\ and \cifartwo\ datasets and evaluate our methods in this setting. We include the experimental results of \cifartwo\ in Table \ref{table: real_noise}, which indicates that our methods can still outperform all the baseline methods in the presence of realistic labeling errors. For \cifarone, it turns out that the real human labeling errors have very little influence on the model performance and thus all of these methods only boost the performance marginally, which is thus not shown here. 


\begin{table*}[t]
\caption{Test accuracy of \cifar\ dataset with real human labeling errors}\label{table: real_noise}
\small
\centering
\begin{tabular}[!h]{>{\arraybackslash}p{1.5cm}|>{\centering\arraybackslash}p{1.5cm}} \hline\hline
 & \cifartwo \\\hline
Base model  & 49.33\\
\random & 57.54\\
\certain & 57.48\\
\uncertain & 56.23\\
\finetune &51.24\\
\ta  & 44.50\\
\craige & 57.92\\
\ourmethodone-K& 59.24\\ \hline
\ourmethodone &\textbf{59.25} \\
\ourmethodtwo &\underline{59.25}\\ \hline\hline
\end{tabular}
\end{table*}

\subsubsection{Effect of the number of meta samples} We further study the effect of the number of meta samples on the performance of our methods by continuously adding more and more meta samples. Specifically, for \cifarone\ dataset with 60\% adversarial label errors, we repetitively add 10 meta samples and run the meta re-weighting algorithm for 4 times after the warm-up phase. The results are included in Table \ref{table: repeat_run} and we only include the baseline methods which perform relatively better than other methods, e.g., \certain\ and \ta. As this table shows, our methods can consistently outperform (with performance gains up to~13\% when 20 meta samples are selected) with respect to the baseline methods.

\subsubsection{Effect of the number of sampled gradients for \ourmethodone\ and \ourmethodtwo} Recall that in Section \nameref{sec: solve_mco}, it is not possible to collect all the calculated gradients from all the iterations during the meta-reweighting training phase to evaluate MCO due to limited GPU memory. This thus motivates the idea of randomly sampling calculated gradients in \ourmethodone\ and \ourmethodtwo. We therefore studied the effect of the number of sampled gradients, i.e., the value of $K$, on the performance of our methods. Specifically, we vary $k$ between 4 and 100 on \ourmethodone\ and conduct meta-reweighting on \mnist\ dataset with 80\% adversarial labeling errors (with the same experimental setup as Section \nameref{sec: experiment_setup}). The results are summarized in Table \ref{table: mnist_rbc_varied_k}. According to this table, we can know that 
the test performance of extremely small $K$ is significantly worse than that of large $K$ (e.g, $K=4$ VS $K>10$), which thus indicates that more gradients would boost the model performance. However, when $K$ is larger than certain value, say $K > 10$ in Table \ref{table: mnist_rbc_varied_k}, no significant performance gains occur while more GPU memory is needed for containing more sampled gradients. This thus indicates that with proper value of $K$, randomly sampling gradients could perfectly balance the test performance and the GPU memory consumption.

\subsubsection{Effect of pre-training phase} We further studied the effect of the pre-training phase on the performance. When pre-training phase is not executed, we select all samples once by using our methods or the baseline methods. We include the results on \cifartwo\ dataset in Table \ref{table: pretraining}, in which we compare our methods against \certain\ and \uncertain\ since they are two relatively better baseline methods in comparison to others. 
According to this table, we can tell that both our methods and baseline methods can generally benefit from pre-training phase, thus demonstrating the benefit of the pre-training phase.

\begin{table}[ht]
\caption{Test accuracy of \ourmethodone\ on \mnist\ dataset with varied $K$}\label{table: mnist_rbc_varied_k}
\small
\centering
\begin{tabular}[!h]{c|c} \hline\hline
 $K$ & Test Accuracy \\ \hline
100 &93.01 \\
50 & 93.32\\
10 & 93.10\\
6 & 91.48\\
4 &91.68 \\\hline
\end{tabular}
\end{table}

\subsubsection{Supplemental results on class-imbalanced + label noise experiments} In addition to evaluating our methods on noisily labeled and imbalanced data separately, we also look at the combination of the two in Table (\ref{Table: cifar_imb_noise}). In these experiments, we perform an initial warmup step with 10 and 100 randomly selected meta samples and then select an additional 100 and 200 samples using the different methods for \cifarone\ and \cifartwo\ respectively. With 40\% uniform noise and class imbalance levels of 200 and 100, our methods outperform all baselines for \cifartwo\ and \cifarone\ with imbalance of 100.

\begin{table}[ht]
\caption{Test accuracy of ResNet-34 on imbalanced \cifarone\ and \cifartwo\ dataset with 40\% uniform noise}\label{Table: cifar_imb_noise}
\small
\centering
\begin{tabular}[!h]{>{\arraybackslash}p{2cm}|>{\centering\arraybackslash}p{1.8cm}>{\centering\arraybackslash}p{1.8cm}|>{\centering\arraybackslash}p{1.8cm}>{\centering\arraybackslash}p{1.8cm}} \hline\hline
 Dataset &\multicolumn{2}{c|}{\cifarone}&\multicolumn{2}{c}{\cifartwo} \\ \hline
Imbalance & 200 & 100 & 200 & 100 \\ \hline
Base Model     & 27.3 & 31.92 & 6.97 & 9.13\\
Random         & \underline{29.32} & 35.49 & 9.30 & 10.30\\
Uncertainty    & 27.91 & 35.59 & 9.16 & 10.62\\
Certainty      & 28.48 & 35.29 & \underline{9.40} & 10.27\\
Finetune       & 28.97 & 36.76 & 8.07 & 7.93\\
TA-VAAL        & \textbf{29.77} & \underline{36.89} & 9.00 & 10.33\\
\hline
\ourmethodone\ & 27.73 & \textbf{38.01} &  \textbf{9.99}& \underline{11.02}\\
\ourmethodtwo\ & 28.44 & 35.75 & 8.17 & \textbf{11.24}\\\hline\hline
\end{tabular}
\end{table}

\begin{table*}[t]
\caption{Test performance with and without pre-training with 60\% label noise}\label{table: pretraining}
\small
\centering
\begin{tabular}[!h]{>{\centering\arraybackslash}p{2cm}|>{\centering\arraybackslash}p{1.6cm}>{\centering\arraybackslash}p{1.6cm}|>{\centering\arraybackslash}p{1.6cm}>{\centering\arraybackslash}p{1.6cm}} \hline\hline
 \multirow{2}{*}{Label noise type} &\multicolumn{2}{c}{uniform}&\multicolumn{2}{|c}{adversarial} \\ \hhline{~----}
 &\makecell{with\\pretraining}&\makecell{without\\pretraining}&\makecell{with\\pretraining}&\makecell{without\\pretraining}\\ \hline
\certain  & 45.95&45.32&47.06&44.35\\
\uncertain & 36.67&35.99& 44.65&44.54\\
\ourmethodone-K &49.32&\underline{47.65}& 49.51&47.66\\ \hline
\ourmethodone & \underline{49.56} &45.27 &\underline{50.60}&\underline{48.91}\\
\ourmethodtwo & \textbf{50.88} &\textbf{50.24}&\textbf{53.14}&\textbf{52.55}\\ \hline\hline
\end{tabular}
\end{table*}

\subsection{Qualitative results}

Figure \ref{fig: visual_cluster} is a visualization using t-SNE \citesupp{van2008visualizing} of the partial sample-wise gradients collected by \ourmethodone\ as well as the cluster centroids generated by the weighted K-means clustering. As this figure shows, there exists obvious clustering structure on the sample-wise gradients, thus justifying the use of the K-means clustering algorithm. In addition, recall that in Section \nameref{sec: solve_mco}, we tailor the vanilla K-means clustering algorithm for solving MCO. Figure \ref{fig: visual_cluster} thus demonstrates the effectiveness of this tailored algorithm since the cluster centroids discovered in this manner cover all the clusters very well.

\begin{figure}
\centering
\includegraphics[width=0.4\textwidth]{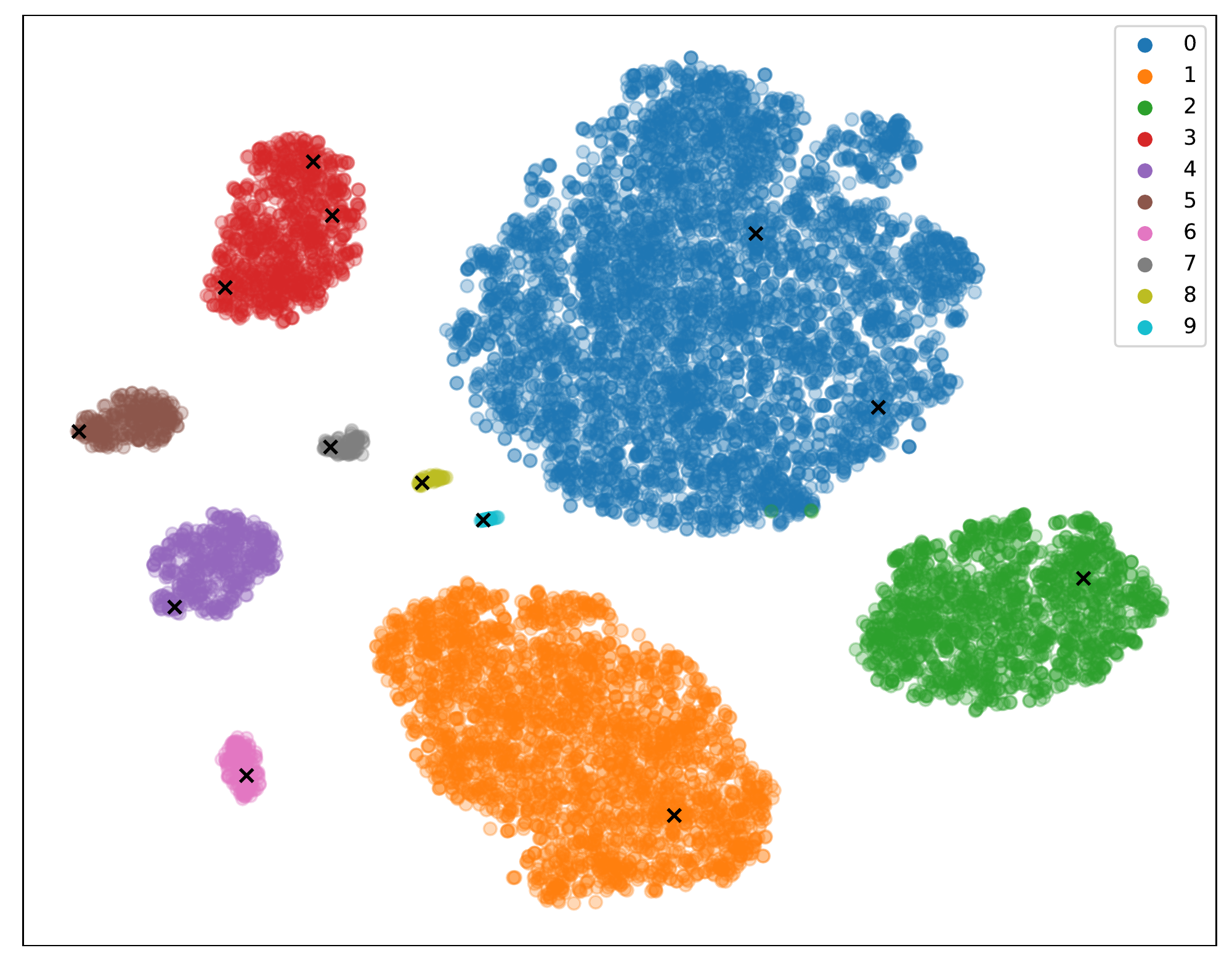}
\caption{Visualization of the partial gradients used in \ourmethodone\ on \cifarone\ with imbalance factor 200 and the corresponding cluster centroids identified by the weighted K-means clustering algorithm.}
\label{fig: visual_cluster}
\end{figure}

\section{Limitations of our work}\label{appendix: limitation}
Most of the theoretical results we present hold in a general setting except some dataset-dependent assumptions (e.g., the assumption of Theorem \ref{theorem: main}). We therefore would love to explore whether those assumptions hold for general datasets or not.

Also, our empirical claims are shown on just three standard datasets for the noisy label case and two datasets for the class imbalance case. The datasets we use are \mnist, \cifarone, and \cifartwo\ which are benchmark image datasets, which, however, do not include the datasets from other domains. Similarly, our empirical results use a ResNet model which is a standard neural network model for vision. We leave the evaluation of our methods on other models like Transformers for future work. 

We also manually introduced label noise and a class imbalance into these datasets and have not yet evaluated on a dataset which is known to contain label noise or a class imbalance. Additionally, we use two standard types of label noise in our experiments, but our method may perform differently if the distribution of noise in another dataset is significantly different from our uniform or adversarial noise. 


Furthermore, the selections of meta samples or validation samples could occur in many different scenarios, e.g., general meta learning framework \citesupp{andrychowicz2016learning} and the data valuation methods depending on a clean validation samples (see the discussion in Section \nameref{sec: related_work}). Therefore, it would be interesting to explore whether the techniques proposed in this paper could address the validation sample selection problem in more general set-up.


\section{Discussions on societal impacts}\label{appendix: social_impact}
Our framework can be useful for domains in which data is commonly noisy and imbalanced. In these settings, the sample labels are usually manually cleaned.
However, labelling in these domains is expensive and may suffer from the risk of getting incorrect labels. 
As indicated by the experiments, our methods can perform very well when the labeling budget is very small. This thus suggests that our methods can reduce the number of required perfect labels. As a consequence, more labeling effort can be spent on the few pivotal samples to get high-quality labels, rather than a large amount of low-quality labels in these domains.

\bibliographystylesupp{plainnat}
\bibliographysupp{reference2}

\end{document}